\documentclass{article}
\usepackage{times}
\usepackage{graphicx} % more modern
\usepackage{subfigure} 

\usepackage{natbib}

\usepackage{algorithm}
\usepackage{algorithmic}

\renewcommand{\algorithmicrequire}{\textbf{Initialize:}}   
\renewcommand{\algorithmicensure}{\textbf{Iteration:}} 
\renewcommand{\algorithmicinput}{\textbf{Input:}} 
\renewcommand{\algorithmicoutput}{\textbf{Output:}}

\usepackage{hyperref}

\usepackage{helvet}
\usepackage{courier}

\usepackage{makecell}
\usepackage{graphicx}
\usepackage{subfigure}
\usepackage{color}
\usepackage{amsfonts}
\usepackage{amsmath}
\usepackage{amssymb}

\usepackage{multirow}
\usepackage{booktabs}

\def\tr{\mbox{tr}}

\def\mB{{\mathcal B}}

\def\mD{{\mathcal D}}

\def\mH{{\mathcal H}}
\def\mI{{\mathcal I}}

\def\mL{{\mathcal L}}
\def\mM{{\mathcal M}}
\def\mN{{\mathcal N}}
\def\mO{{\mathcal O}}

\def\mY{{\mathcal Y}}

\DeclareMathAlphabet\mathbfcal{OMS}{cmsy}{b}{n}
\def\0{{\bf 0}}
\def\1{{\bf 1}}

\def\bA{{\bf A}}

\def\bD{{\bf D}}

\def\bI{{\bf I}}

\def\bN{{\bf N}}

\def\bS{{\bf S}}

\def\bV{{\bf V}}

\def\bX{{\bf X}}
\def\bY{{\bf Y}}
\def\bZ{{\bf{Z}}}

\def\be{{\bf e}}

\def\bn{{\bf n}}

\def\bx{{\bf x}}
\def\by{{\bf y}}
\def\bz{{\bf z}}

\def\mmE{{\mathbb E}}

\def\mmR{{\mathbb R}}

\def\mmZ{{\mathbb Z}}

\def\trsp{{\sf T}}

\def\bx{{\bf x}}
\def\bX{{\bf X}}
\def\by{{\bf y}}
\def\bY{{\bf Y}}

\def\bz{{\bf z}}

\def\tr{\mathrm{tr}}

\usepackage{ntheorem}

\newtheorem{deftn}{Definition}
\newtheorem{thm}{Theorem}
\newtheorem*{*thm}{Theorem}
\newtheorem{prop}{Proposition}
\newtheorem{lemma}{Lemma}
\newtheorem*{*lemma}{Lemma}

\newenvironment*{proof}{\textbf{Proof}\quad}{\hfill $\square$\par}

\usepackage{xspace}
\newcommand*{\eg}{e.g.\@\xspace}
\newcommand*{\ie}{i.e.\@\xspace}

\usepackage[accepted]{icml2024}

\usepackage{comment}
\usepackage{soul}

\icmltitlerunning{Deep Regression Representation Learning with Topology}

\begin{document}

\twocolumn[
\icmltitle{Deep Regression Representation Learning with Topology}

% It is OKAY to include author information, even for blind
% submissions: the style file will automatically remove it for you
% unless you've provided the [accepted] option to the icml2023
% package.

% List of affiliations: The first argument should be a (short)
% identifier you will use later to specify author affiliations
% Academic affiliations should list Department, University, City, Region, Country
% Industry affiliations should list Company, City, Region, Country

% You can specify symbols, otherwise they are numbered in order.
% Ideally, you should not use this facility. Affiliations will be numbered
% in order of appearance and this is the preferred way.
\icmlsetsymbol{equal}{*}

\begin{icmlauthorlist}
\icmlauthor{Shihao Zhang}{soc}
\icmlauthor{Kenji Kawaguchi}{soc}
\icmlauthor{Angela Yao}{soc}
\end{icmlauthorlist}

\icmlaffiliation{soc}{National Unviersity of Singapore}

\icmlcorrespondingauthor{Shihao Zhang}{zhang.shihao@u.nus.edu}

% You may provide any keywords that you
% find helpful for describing your paper; these are used to populate
% the "keywords" metadata in the PDF but will not be shown in the document
\icmlkeywords{Machine Learning, ICML}

\vskip 0.3in
]

% this must go after the closing bracket ] following \twocolumn[ ...

% This command actually creates the footnote in the first column
% listing the affiliations and the copyright notice.
% The command takes one argument, which is text to display at the start of the footnote.
% The \icmlEqualContribution command is standard text for equal contribution.
% Remove it (just {}) if you do not need this facility.

\printAffiliationsAndNotice{}  % leave blank if no need to mention equal contribution
% \printAffiliationsAndNotice{\icmlEqualContribution} % otherwise use the standard text.

\begin{abstract}

Most works studying representation learning focus only on classification and neglect regression. Yet, the learning objectives and, therefore, the representation topologies of the two tasks are fundamentally different: classification targets class separation, leading to disconnected representations, whereas regression requires ordinality with respect to the target, leading to continuous representations. We thus wonder how the effectiveness of a regression representation is influenced by its topology, with evaluation based on the Information Bottleneck (IB) principle. The IB principle is an important framework that provides principles for learning effective representations. We establish two connections between it and the topology of regression representations. The first connection reveals that a lower intrinsic dimension of the feature space implies a reduced complexity of the representation $\bZ$. This complexity can be quantified as the conditional entropy of $\bZ$ on the target $\bY$, and serves as an upper bound on the generalization error. The second connection suggests a feature space that is topologically similar to the target space will better align with the IB principle. Based on these two connections, we introduce PH-Reg, a regularizer specific to regression that matches the intrinsic dimension and topology of the feature space with the target space. Experiments on synthetic and real-world regression tasks demonstrate the benefits of PH-Reg. Code: \url{https://github.com/needylove/PH-Reg}.

\end{abstract}

\section{Introduction}
\label{sec:introduction}
Regression is a fundamental task in machine learning in which input samples are mapped to a continuous target space. Representation learning empowers models to automatically extract, transform, and leverage relevant information from data, leading to improved performance. The information bottleneck (IB) principle \citep{shwartz2017opening} provides a theoretical framework and guiding principle for learning effectiveness representations. The IB principle suggests learning a representation $\bZ$ with sufficient information about the target $\bY$ but minimal information about the input $\bX$. For representation $\bZ$, sufficiency keeps all necessary information on $\bY$, while the minimality reduces $\bZ$'s complexity and prevents overfitting. The optimal representation, as specified by~\citet{achille2018information, achille2018emergence}, is the most useful (sufficient) and minimal. Yet, they specified classification and neglect regression.

The IB principle is applicable to both classification and regression in that both learn minimal and sufficient representations. However, there are some fundamental differences. For example, classification shortens the distance between features belonging to the same class while elongating the distance between features of different classes; the shortening and elongating of distances can be interpreted as minimality and sufficiency, respectively~\citep{boudiaf2020unifying}. The two effects lead to disconnected representations \citep{brown2022verifying}. By contrast, in regression, the representations are shown to be continuous, connected and form an ordinal relationship with respect to the target \citep{zhang2023improving}. The disconnected and connected representations are topologically different, as they have different $0^{th}$ Betti numbers. The $0^{th}$ Betti number represents the connectivity in topology, influencing the `shape' of the feature space\footnote{In this work, the feature space represents the set of projected data points, \ie the manifold, rather than the entire ambient space.}. While there are a few works investigating the influence of the representation topology in classification \cite{hofer2019connectivity,chen2019topological}, regression is overlooked. We thus wonder what topology the feature space should have for effective regression and how the topology of the feature space is connected to the IB principle. 

In this work, we establish two connections between the topology of the feature space and the IB principle for regression representation learning in deep learning. To establish the connections, we first demonstrate that minimizing the conditional entropies $\mH(\bY|\bZ)$ and $\mH(\bZ|\bY)$ can better align with the IB principle. The entropy of a random variable reflects its uncertainty. Specifically, for regression, the conditional entropy $\mH(\bZ|\bY)$ is linked to the minimality of $\bZ$ and serves as an upper-bound on the generalization error.

The first connection reveals that $\mH(\bZ|\bY)$ is bounded by the intrinsic dimension (ID) of the feature space, which suggests encouraging a lower ID feature space for better generalization ability. However, the ID of the feature space should not be less than the ID of the target space to guarantee sufficient representation capabilities. Thus, a feature space with ID equals the target space is desirable.
The intrinsic dimension (ID) is a fundamental property of data topology. Intuitively, it can be regarded as the minimal number of dimensions to describe the representation without significant information loss~\citep{ansuini2019intrinsic}.

The second connection reveals that having a representation $\bZ$ homeomorphic to the target space $\bY$ is desirable when both $\mH(\bY|\bZ)$ and $\mH(\bZ|\bY)$ are minimal. The homeomorphism between two spaces can be described intuitively as the continuous deformation of one space to the other. From a topological viewpoint, two spaces are considered the same if they are homeomorphic \cite{hatcher2005algebraic}. However, directly enforcing homeomorphism can be challenging to achieve since the representation $\bZ$ typically lies in a high-dimensional space that cannot be modeled without sufficient data samples. As such, we opted to enforce the topological similarity between the target and feature spaces.  Here, topological similarity refers to the similarity in topological features, such as clusters and loops, and their localization \cite{trofimov2023learning}.

These connections naturally inspire us to learn a regression feature space that is topologically similar to and has the same intrinsic dimension as the target space. To this end, we introduce a regularizer called Persistent Homology Regression Regularizer (PH-Reg).
In classification, interest has grown in regulating the intrinsic dimension. For instance, \citet{zhu2018ldmnet} explicitly penalizes intrinsic dimension as regularization, while \citet{ma2018dimensionality} uses intrinsic dimensions as weights for noise label correction. However, a theoretical justification for using intrinsic dimension as a regularizer is lacking, and they overlook the topology of the target space.  Experiments on various regression tasks demonstrate the effectiveness of PH-Reg.
Our main contributions are three-fold:

\begin{itemize}
    \item We are the first to investigate effective feature space topologies for regression. 
    We establish novel connections between the topology of the feature space and the IB principle, which also provides justification for exploiting intrinsic dimension as a regularizer.
    \item Based on the IB principle, we demonstrate that $\mH(\bZ|\bY)$ serves as an upper-bound on the generalization error in regression, providing insights for enhancing generalization ability.
    \item We introduce a regularizer named PH-Reg based on the established connections.  Applying PH-Reg achieves significant improvement in coordinate prediction on synthetic datasets and real-world regression tasks such as super-resolution, age estimation, and depth estimation.
\end{itemize}

\section{Related Works}

\textbf{Intrinsic dimension}. 
Raw data and learned data representations often lie on lower intrinsic dimension manifolds but are embedded within a higher-dimensional ambient space~\citep{bengio2013representation}. The intrinsic dimension of the feature space from the last hidden layer has shown a strong connection with the network generalization ability \citep{ansuini2019intrinsic}, and several widely used regularizers like weight decay and dropout effectively reduce the intrinsic dimension~\citep{brown2022relating}. Commonly, the generalization ability increases with the decrease of the intrinsic dimension. However, a theoretical justification for why this happened is lacking, and our established connections provide an explanation for this phenomenon in regression.

The intrinsic dimension can be estimated by methods such as the TwoNN \citep{facco2017estimating} and Birdal's estimator \citep{birdal2021intrinsic}. Among the relevant studies, \cite{birdal2021intrinsic} is the most closely related to ours.
This work demonstrates that the generalization error can be bounded by the intrinsic dimension of training trajectories, which possess fractal structures. However, their analysis is based on the parameter space, while ours is on the feature space. Furthermore, we take the target space into consideration, ensuring sufficient representation capabilities.

\textbf{Topological data analysis}. Topological data analysis is a recent field that provides a set of topological and geometric tools to infer robust features for complex data \cite{chazal2021introduction}. It can be coupled with feature learning to ensure that learned representations are robust and reflect the training data's underlying topology and geometric information \cite{rieck2020uncovering}. It has benefitted diverse tasks ranging from fMRI data analysis~\citep{rieck2020uncovering} to and AI-generated text detection  \citep{tulchinskii2023intrinsic}.  It can also be used as a tool to compare data representations~\citep{barannikov2021representation} and data manifolds~\citep{barannikov2021manifold}. To learn representations that reflect the topology of the training data, a common strategy is to preserve different dimensional topologically relevant distances of the input space and the feature space \citep{moor2020topological,trofimov2023learning}. We follow \citet{moor2020topological} to preserve topology information. However, unlike classification, regression's target space is naturally a 
metric space rich in topology induced by the metric, and crucial for the intended task. Consequently, we leverage the topology of the target space, marking the first exploration of topology specific to effective representation learning for regression.

\section{Learning a Desirable Regression Representation}
\label{Sec:thm}

From a topology point of view, what topological properties should a representation for regression have? More simply put, what `shape' or structure should the feature space have for effective regression? In this work, we suggest a desirable regression representation should (1) have a feature space topologically similar to the target space and (2) the intrinsic dimension of the feature space should be the same as the target space. We arrive at this conclusion by establishing two connections between the topology of the feature space and the Information Bottleneck principle. 

Below, we first introduce the notations in Sec. \ref{subsection:Notation} and connect the IB principle with two terms $\mH(\bZ|\bY)$ and $\mH(\bY|\bZ)$ in Sec. \ref{subsection:IB}. We then demonstrate that $\mH(\bZ|\bY)$  is the upper-bound on the generalization error in regression in Sec. \ref{subsection:GeneralizationError}. This later provides justification for why lower ID implies higher generalization ability. Subsequently, we establish the first connection in Sec. \ref{subsection:Connection1}, revealing that $\mH(\bZ|\bY)$ is bounded by the ID of the feature space. Finally, we establish the second connection, the topological similarity between the feature and target spaces, in Sec. \ref{subsection:Connection2}. Two motivating examples are provided in Sec. \ref{subsection:Example} to enhance understanding of the two connections intuitively.

\subsection{Notations}
\label{subsection:Notation}
Consider a dataset $S=\{\bx_i,\bz_i, \by_i\}_{i=1}^N$ with $N$ samples, sampled from a distribution $P$ with the corresponding label $\by_i \in \mY$. To predict $\by_i$, a neural network first encodes the input $\bx_i$ to a representation $\bz_i \in \mmR^d$ before apply a regressor $f$, \ie $\hat{\by}_i = f(\bz_i)$. The encoder and the regressor $f$ are trained by minimizing a task-specific regression loss $\mL_m$ based on a distance between $\hat{\by}_i$ and $\by_i$, \ie $\mL_m = g(||\hat{\by}_i -\by_i||_2)$. Typically, an L2 loss is used, \ie $\mL_m = \frac{1}{N}\sum_i||\hat{\by}_i -\by_i||_2$, though more robust variants exist such as L1 or the scale-invariant error \citep{eigen2014depth}. Note that the dimensionality of $\by$ is task-specific and is not limited to 1. We denote $\bX, \bY$, and $\bZ$ as random variables representing $\bx, \by$, and $\bz$, respectively.

\subsection{$\mI\mB$ purely between $\bY$ and $\bZ$}
\label{subsection:IB}

The IB tradeoff is a practical implementation of the IB principle in machine learning. It suggests that a desirable $\bZ$ should contain sufficient information about the target $\bY$ (\ie, maximize the mutual information $\mI(\bZ; \bY)$) and minimal information about the input $\bX$ (\ie, minimize $\mI(\bZ; \bX)$). The trade-off between the two aims is typically formulated as a minimization of the associated Lagrangian, $\mI\mB := - \mI(\bZ; \bY) + \beta \mI(\bZ; \bX)$, where $\beta>0$ is the Lagrange multiplier.

To establish the connections, we first formulate the IB tradeoff into relationships purely between $\bY$ and $\bZ$. The following theorem shows that minimizing the conditional entropies  $\mH(\bY|\bZ)$ and $\mH(\bZ|\bY)$ can be seen as a proxy for optimizing the IB tradeoff when $\beta\in(0,1)$:

\begin{thm}
\label{thm:bottleneck}
Assume that the conditional entropy $\mH(\bZ|\bX)$ is a fixed constant for $\bZ \in \mathcal Z$ for some set  $\mathcal Z$ of the random variables, or that $\bZ$ is deterministic given $\bX$. Then,  $\min_{\bZ} ~ \mI\mB= \min_{\bZ} ~ \{(1-\beta) \mH(\bY|\bZ) + \beta \mH(\bZ|\bY) \}$.
\end{thm}

The detailed proof of Theorem \ref{thm:bottleneck} is provided in Appendix \ref{appendix:thm:bottleneck}. Here, we provide a brief overview by decomposing the terms. The conditional entropy $\mH(\bY|\bZ)$ encourages the learned representation $\bZ$ to be informative about the target variable $\bY$.
When considering $\mI(\bZ; \bY)$ as a signal, the term $\mH(\bZ|\bY)$ in Theorem \ref{thm:bottleneck} can be thought of as noise, since $\mI(\bZ; \bY) = \mH(\bZ) - \mH(\bZ|\bY)$ and $\mH(\bZ)$ represents the total information. Consequently, minimizing $\mH(\bZ|\bY)$ can be seen as learning a minimal representation by reducing noise. The minimality can reduce the complexity of $\bZ$ and prevent neural networks from overfitting \cite{tishby2015deep}.

It is worth mentioning that the fixed constant assumption given in Theorem \ref{thm:bottleneck} holds for most neural networks, as neural networks are commonly deterministic functions. For stochastic representations, we commonly learn a distribution $p(\bZ|\bX)$ approaching a fixed distribution, like the standard Gaussian distribution in VAE. In this case, $\mH(\bZ|\bX)$ will tend to be a fixed constant for $\bZ$. Discussions about the choice of $\beta$, \ie $\beta\in(0,1)$ or $\beta >1$, and more illustrations are given in Appendix \ref{appendix:explainTheorem1}.

\subsection{$\mH(\bZ|\bY)$ upper-bound on the generalization error}
\label{subsection:GeneralizationError}
Next, we show $\mH(\bZ|\bY)$ upper-bound on the generalization error.

\begin{thm}
\label{thm:bound}
Consider dataset $S = \{\bx_i,\bz_i, \by_i\}^N_{i=1}$ sampled from distribution $P$, where $\bx_i$ is the input, $\bz_i$ is the corresponding representation, and $\by_i$ is the label. Let  $d_{max} = \max_{\by\in \mY} \min_{\by_i \in S} ||\by-\by_i||_2$  be the maximum distance of $\by$ to its nearest $\by_i$. Assume  $(\bZ|\bY=\by_i)$ follows a distribution $\mD$ and 
the dispersion of $\mD$ is bounded by its entropy:
\begin{equation}
    \mmE_{\bz \sim \mD}[||\bz - \bar{\bz}||_2] \leq Q(\mH(\mD)),
\end{equation}
where  $\bar{\bz}$ is the mean of the distribution $\mD$ and $Q(\mH(\mD))$ is some function of 
$\mH(\mD)$. 
Assume the regressor f is $L_1$-Lipschitz continuous, then as $d_{max} \rightarrow 0$, we have:
%~\AY{XX}
\begin{align} \label{eq:1}
    &  \mmE_{\{\bx, \bz, \by\} \sim P} [||f(\bz) - \by||_2]  \\
    &\leq  \mmE_{\{\bx, \bz, \by\} \sim S}(||f(\bz) -\by||_2) 
    + 2L_1 Q(\mH(\bZ|\bY))
\end{align}
\end{thm}
The detailed proof of Theorem \ref{thm:bound} is provided in Appendix \ref{appendix:thm:bound}, and a comparison to a related bound is given in Appendix \ref{appendix:compareBound}. Theorem \ref{thm:bound} states that the generalization error $|\mmE_P[||f(\bz) - \by||_2]-\mmE_S[||f(\bz) - \by||_2]|$, defined as the difference between the population risk $\mmE_P[||f(\bz) - \by||_2]$ and the empirical risk $\mmE_S[||f(\bz) - \by||_2]$, is bounded by the $\mH(\bZ|\bY)$ in Theorem \ref{thm:bottleneck}. Theorem \ref{thm:bound} suggests minimizing $\mH(\bZ|\bY) $ will improve generalization performance.

The tightness of the bound in Theorem \ref{thm:bound} depends on the function $Q$, which aims to bound the dispersion (\ie, $\mmE_{\bz \sim \mD}[||\bz - \bar{\bz}||_2]$) of a distribution by its entropy. 
For a given distribution $\mD$, $Q$ exists when its dispersion and entropy are bounded, as we can find a $Q$ to scale its entropy larger than its dispersion in this case. Proposition \ref{prop_supportTheorem1} provides examples of the function $Q$ for various distributions, and the corresponding proof is provided in Appendix \ref{appendix:thm:bound}.

\begin{prop}
\label{prop_supportTheorem1}
    If $\mD$ is a multivariate normal distribution $\mN(\bar{\bz}, \Sigma = k\bI)$, where $k>0$ is a scalar and $\bar{\bz}$ is the mean of the distribution $\mD$. Then, the function $Q(\mH(\mD))$ in Theorem \ref{thm:bound} can be selected as $Q(\mH(\mD)) = \sqrt{\frac{d(e^{2\mH(\mD)})^{\frac{1}{d}}}{2\pi e}}$, where $d$ is the dimension of $\bz$. If $\mD$ is a uniform distribution, then the $Q(\mH(\mD))$ can be selected as $Q(\mH(\mD)) = \frac{e^{\mH(\mD)}}{\sqrt{12}}$.
\end{prop}

\begin{figure*}[!t]
\centering
\subfigure[Lower intrinsic dimension] {
 \label{fig_intro_id}
\includegraphics[width=0.49\linewidth, height=0.5\columnwidth]{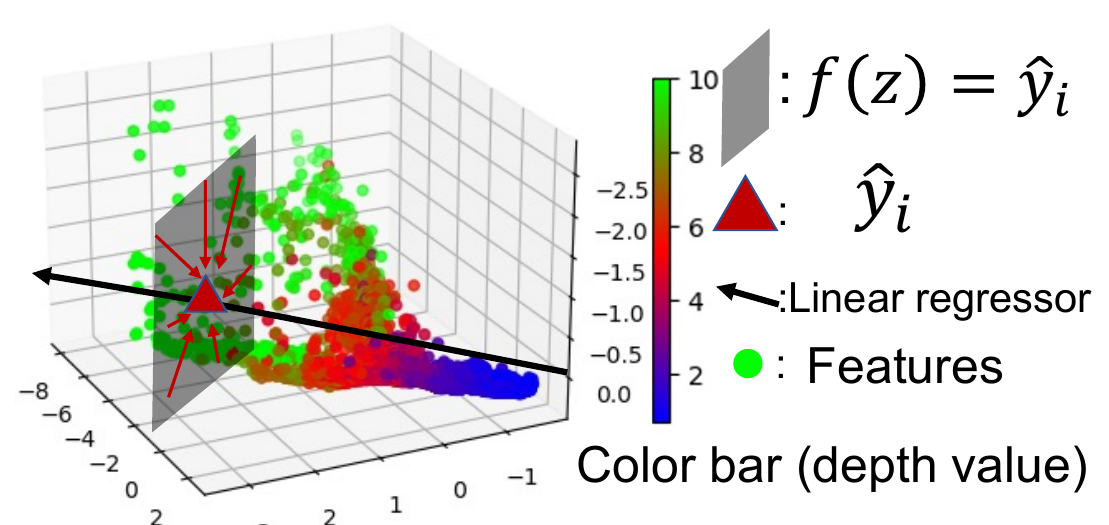}
}
\subfigure[Topology similarity ] {
 \label{fig_intro_topo}
\includegraphics[width=0.45\linewidth, height=0.5\columnwidth]{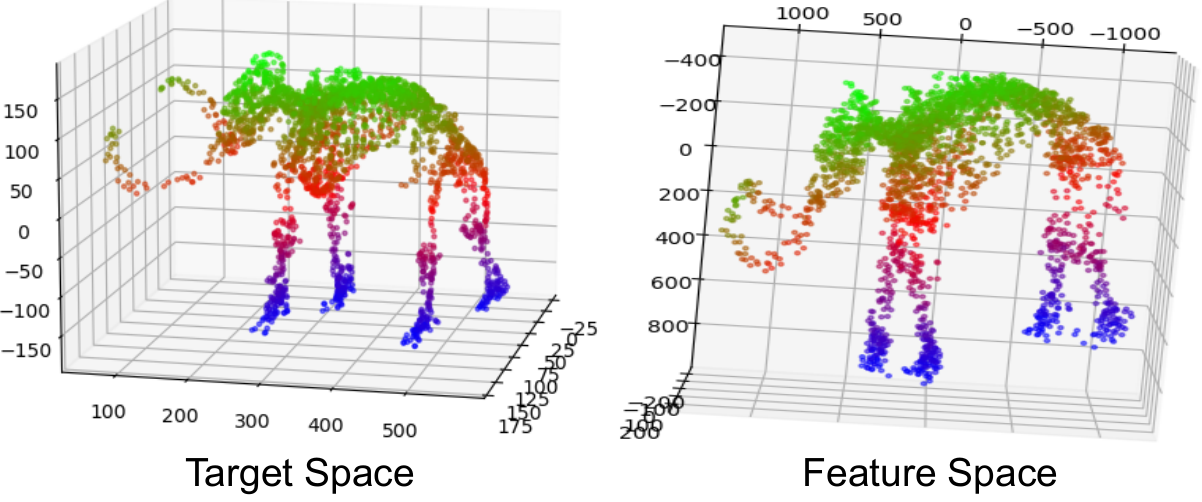}
}
\caption{(a) Visualization of the feature space from depth estimation task.
Enforcing an ID equal to the target space ($1$ dimensional) will squeeze the feature space into a line, reducing the unnecessary $\mH(\bZ|\bY=\by_i)$ corresponding to the solution space of $f(\bz) = \hat{\by}_i$ (the gray quadrilateral) for all $i$ and implying a lower $\mH(\bZ|\bY)$.
(b) Visualization of the feature space (right) and the `Mammoth' shape target space (left), see Sec. \ref{subsection:synthetic} for details. The feature space is topologically similar to the target space 
.}
\label{fig_motivation1}
\end{figure*}

\subsection{Motivating Examples}
\label{subsection:Example}

\textbf{Encouraging the same intrinsic dimension}. Figure \ref{fig_intro_id} plots pixel-wise representations of the last hidden layer's feature space, depicted as dots with different colors corresponding to ground truth depth. These representations are obtained from a batch of $32$ images from the NYU-v2 test set for depth estimation. A modified ResNet-50 produces these representations, with the last hidden layer changed to dimension 3 for visualization.
This figure provides a visualization of the last hidden layer's feature space, where the representations lay on a manifold where the ID varies locally from $1$ (blue region) to $3$ (green region). The black arrow represents the linear regressor's weight vector $\theta$, and the predicted depth $\hat{\bY} = f(\bZ) = \theta^{\trsp} \bz$ is obtained by mapping $\bZ$ (represented as dots) to $\theta$. The gray plane represents the solution space of $f(\bZ) = \hat{\by}_i$, 
and the entropy of its distribution in this plane, \ie $\mH(\bZ|\hat{\bY}=\hat{\by}_i)$, can be seen as an approximate of $\mH(\bZ|\bY=\by_i)$.

The target space for depth estimation is one-dimensional; enforcing an intrinsic dimension to match the 1D target space will squeeze the feature space into a line. Under such a scenario, 
the solution space of $f(\bz) = \hat{\by}_i$ is compressed into a point, implying $\mH(\bZ|\hat{\bY}=\hat{\by}_i)=0$ (discrete case) and a lower $\mH(\bZ|\bY=\by_i)$. Lower $\mH(\bZ|\bY=\by_i)$ for all $i$ implies a lower $\mH(\bZ|\bY)$. \textit{Thus, by controlling the ID, we obtain a lower $\mH(\bZ|\bY)$, implying a higher generalization ability.}
Since the ID of the feature space is commonly higher than the ID of the target space, the first connection generally encourages learning a lower ID feature space.

In classification, we tighten clusters for a lower $\mH(\bZ|\bY)$, while in regression, lowering the ID achieves a lower $\mH(\bZ|\bY)$.
Lowering the ID of feature space can be intuitively understood as tightening the clusters in classification, where each solution space represents a cluster in classification.

\textbf{Enforcing topological similarity}. Figure \ref{fig_intro_topo} provides a PCA visualization (from $100$ dimension to $3$ dimension, t-sne visualization can be found in Figure \ref{fig_syn}) of the 
feature space with a 'Mammoth' shape target space (see Sec. \ref{subsection:synthetic} for details). 
This feature space is topologically similar to the target space, which indicates regression potentially captures the topology of the target space. The second connection suggests improving such similarity.

\subsection{Encouraging the Same Intrinsic Dimension}
\label{subsection:Connection1}

Now, we can establish our first connection, which reveals that $\mH(\bZ|\bY)$ is bounded by the ID of the feature space. Note, intrinsic dimension is not a well-defined mathematical object, and different mathematical definitions exist \cite{ma2018dimensionality, birdal2021intrinsic}.
We first define Intrinsic Dimension following \citet{ghosh2023local}:
\begin{deftn}
\label{deftn:ID}
    (Intrinsic Dimension). We define the intrinsic dimension of the manifold $\mM$ of a random variable $\bX$ as
    \begin{align}
        \text{Dim}_{\text{ID}}\mM = \lim_{\epsilon \rightarrow 0^{+}} \mmE_{\rho \sim p(\bX)}[d_{\epsilon}(\rho)],
    \end{align}
    where $d_{\epsilon}(\rho) = \min n  ~\mathrm{s.t.}~ (\bV_1, \bV_2, \cdots, \bV_n) \leftrightarrow \bX_{\epsilon}^\rho$ can be regarded as the intrinsic dimension locally at point $\rho$ in the manifold $\mM$. $\bV_1, \bV_2, \cdots, \bV_n$ represent random random variables, $\leftrightarrow$ means there exist continuous functions $f_1, f_2$ such that $f_1(\bV_1, \bV_2, \cdots, \bV_n)=X_{\epsilon}^\rho$ and $f_2(\bX_{\epsilon}^\rho)=(\bV_1, \bV_2, \cdots, \bV_n)$.  $\bX_{\epsilon}^\rho$ is a new random variable that follows distribution $P_{\epsilon}^\rho$ given by:
    \begin{align}
        P_{\epsilon}^\rho(\bX) = 
        \begin{cases}
           \frac{P(\bX)}{c}, \quad \text{if} ~ ||\bX -\rho|| \leq \epsilon \\
           0, \quad \text{otherwise}
        \end{cases}
    \end{align}
    where $c = \int_{||\bX -\rho|| \leq \epsilon} {P(\bX)} d\bX$.
\end{deftn}

The manifold is assumed to locally resemble a $n$-dimensional Euclidean space. Intuitively, we can consider the ID as the expectation of $n$ over the distribution of this manifold.

\begin{thm}
\label{thm:entropyToID}
Assume that $\bz$ lies in a manifold $\mM$ and the $\mM_i \subset \mM$ is a 
%\hl{manifold corresponding to $\mH(\bZ|\bY=\by_i)$}. 
manifold corresponding to the distribution $(\bz|\by=\by_i)$.
Let $C(\epsilon)$ be some function of $\epsilon$:
\begin{equation}
C(\epsilon) = \int_{||\bz-\bz'|| \leq \epsilon} P'(\bz) d\bz,
\end{equation}
where $P'(\bz)$ is the probability of $\bz$ when $(\bz|\by=\by_i)$ is uniformly distributed across $\mM_i$, and $\bz'$ is any point on $\mM_i$.
Then, as $\epsilon \rightarrow 0 ^{+}$, 
we have:
\begin{align}
    \mH(\bZ|\bY) &= \mmE_{\by_i \sim \mY} \mH(\bZ| \bY=\by_i) \\
    &\leq \mmE_{\by_i \sim \mY}[-\log(\epsilon)\text{Dim}_{\text{ID}} \mM_i + \log\frac{K}{C(\epsilon)}],
\end{align}
for some fixed scalar K. $\text{Dim}_{\text{ID}}\mM_i$ is the intrinsic dimension of the manifold $\mM_i$.
\end{thm}

Theorem \ref{thm:entropyToID} is derived from [\cite{ghosh2023local}, Proposition 1]. The detailed proof is provided in Appendix \ref{appendix:thm:entropyToID}. Theorem \ref{thm:entropyToID} states that the conditional entropy $\mH(\bZ|\bY)$ is bounded by the IDs of manifolds corresponding to the distribution $(\bz|\by = \by_i)$, and the bound is tight when $(\bz|\by = \by_i)$ are uniformly distributed across the manifolds.

Since $\mM_i \subset \mM$, Theorem \ref{thm:entropyToID} suggests that reducing the intrinsic dimension of the feature space $\mM$ will lead to a lower $\mH(\bZ|\bY)$, 
which in turn implies a better generalization performance based on 
% which implies improved generalization performance according to
Theorem \ref{thm:bound}. On the other hand, the ID of $\mM$ should not be less than the intrinsic dimension of the target space to guarantee sufficient representation capabilities. Thus, a $\mM$ with an intrinsic dimension equal to the dimensionality of the target space is desirable.

\subsection{Enforcing Topological Similarity}
\label{subsection:Connection2}
Below, we establish the second connection: topological similarity between the feature and target spaces. We first define the optimal representation following \citet{achille2018information}.

\begin{deftn}
\label{def:optimalRepresentation}
(Optimal Representation). The representation $\bZ$ is optimal if (1) $\mH(\bY|\bZ) = \mH(\bY|\bX)$ and (2) $\bZ$ is fully determined given $\bY$, \ie $\mH(\bZ|\bY)$ is minimal.
\end{deftn}

In Definition \ref{def:optimalRepresentation}, $\mH(\bY|\bZ) = \mH(\bY|\bX)$ means $\bZ$ is sufficient for the target $\bY$, while $\mH(\bZ|\bY)$ is minimal means $\bZ$ discards all information that is not relevant to $\bY$, and $\bZ$ is fully determined given $\bY$. 
For continuous entropy, a minimal $\mH(\bZ|\bY)$ implies that $\mH(\bZ|\bY) = - \infty$, as $\bZ$ is distributed as a delta function once given $\bY$. In the discrete case, $\mH(\bZ|\bY) = 0$.

\begin{prop}
\label{prop:optimal}
Let the target $\bY =\bY' + \bN'$  
where $\bY'$ is fully determined by $\bX$ and $\bN'$ is 
the aleatoric uncertainty that is independent of $\bX$. Assume the underlying mapping $f'$ from $\bZ$ to $\bY'$ and $g'$ from $\bY'$ to $\bZ$ are continuous, where the continuous mapping is based on the topology induced by the Euclidean distance. Then the representation $\bZ$ is optimal if and only if $\bZ$ is homeomorphic to $\bY'$. 
\end{prop}

The detailed proof of Proposition \ref{prop:optimal} is provided in Appendix \ref{appendix:prop:optimal}. 
Proposition \ref{prop:optimal} demonstrates that the optimal $\bZ$ is homeomorphic to $\bY'$, implying the need to learn a $\bZ$ that is homeomorphic to $\bY'$.
%hl{encouraging $\bZ$ and $\bY'$ to be homeomorphic}.
However, directly enforcing homeomorphism 
% achieving this 
can be challenging to achieve since $\bY'$ is generally unknown, and the representation $\bZ$ typically lies in a high-dimensional space that cannot be modeled without sufficient data samples. 
As such, we opted to enforce the topological similarity between the target and feature spaces, preserving topological features similar to homomorphism. 
Here, topological similarity refers to the similarity in topological features, such as clusters and loops, and their localization \cite{trofimov2023learning}.
The two established connections imply that the desired $\bZ$ should be topologically similar to the target space and share the same ID as the target space. More illustrations are given in Appendix \ref{appendix:discussionProposition2} and \ref{appendix:discussionRegressor}.

\section{PH-Reg for Regression}

Our analysis in Sec.~\ref{Sec:thm} inspires us to learn a feature space that is (1) topologically similar to the target space and (2) with an intrinsic dimension (ID) equal to that of the target space. To this end, we propose a regularizer named 
%~\AY{give full name first} 
Persistent Homology Regression Regularizer (PH-Reg).  PH-Reg features two terms: an intrinsic dimension term $\mL_d$ and a topology term $\mL_t$.
$\mL_d$ follows Birdal's regularizer \citep{birdal2021intrinsic} to control the ID of feature space. Additionally, it considers the target space to ensure sufficient representation capabilities. $\mL_d$ exploit the topology autoencoder \citep{moor2020topological} to encourage the topological similarity. Note the two regularizer terms are mainly introduced to verify our connections, and other ID and topology regularizers can also be considered. 
However, empirical observations suggest that our $\mL_d$ and $\mL_t$ effectively align with our established connections, perform well, and do not conflict with each other.

We first introduce some notations. Let $\bZ_{n}$ represent the set of $n$ samples from $\bZ$, and $\bY_n$ be the labels corresponding to $\bZ_n$. We denote $\text{PH}_{0}(\text{VR}(\bZ_n))$ the $0$-dimensional persistent homology. Intuitively, $\text{PH}_{0}(\text{VR}(\bZ_n))$ can be regarded as a set of edge lengths, where the edges are derived from the minimum spanning tree obtained from the distance matrix $\bA^{\bZ_n}$ of $\bZ_n$. 
We denote $\pi^{\bZ_n}, \pi^{\bY_n}$ the set of the index of edges in the minimum spanning trees of $\bZ_n$ and $\bY_n$, respectively, and $\bA^{*}[\pi^{*}]$ the corresponding length of the edges. Let $E(\bZ_{n}) =  \sum_{\gamma \in \text{PH}_0(\text{VR}(\bZ_{n}))} |I(\gamma)|$ be the sum of edge lengths of the minimum spanning trees corresponding to $\bZ_{n}$. We define $E(\bY_{n})$ similarly. Some topology preliminaries are given in Appendix \ref{appendix:Preliminaries}.

\begin{figure*}[!t]
\centering
\subfigure[Regression] {
 \label{fig_reg2}
\includegraphics[width=0.36\columnwidth,height=0.3\columnwidth]{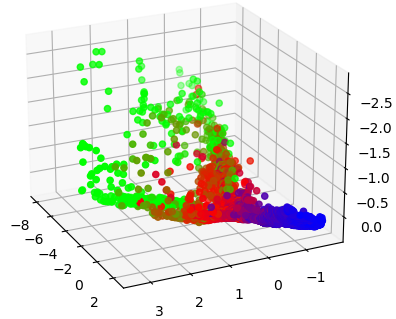}
}
\subfigure[Regression $+\mL'_d$] {
 \label{fig_l'd}
\includegraphics[width=0.36\columnwidth,height=0.3\columnwidth]{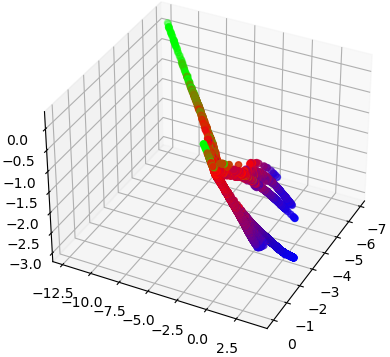}
}
\subfigure[Regression $+\mL_d$] {
 \label{fig_ld}
\includegraphics[width=0.36\columnwidth,height=0.3\columnwidth]{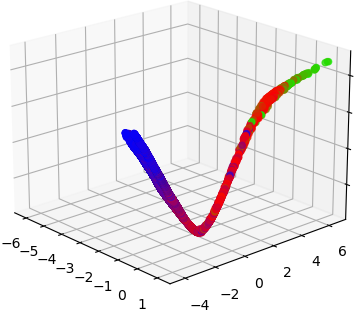}
}
\subfigure[Regression $+\mL_t$] {
 \label{fig_lt}
\includegraphics[width=0.36\columnwidth,height=0.3\columnwidth]{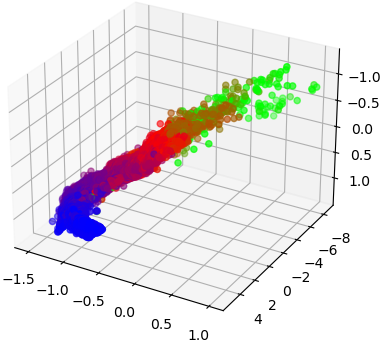}
}
\subfigure[
%~\hl{XX}$
Regression $+ \mL_R$] {
 \label{fig_lr}
\includegraphics[width=0.38\columnwidth,height=0.3\columnwidth]{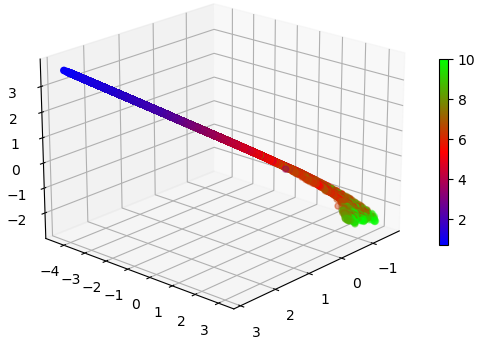}
}
\caption{
Visualization of the last hidden layer's feature space from the depth estimation task. The representations are obtained through a modified ResNet-50, with the last hidden layer changed to dimension 3 for visualization. The target space is a $1$-dimensional line, and colors represent the ground truth depth. (b) $\mL'_d$  encourages a lower intrinsic dimension yet fails to preserve the topology of the target space. (c) $\mL_d$ takes the target space into consideration and can further preserve its topology. (d) $\mL_t$ can enforce the topological similarity between the feature and target spaces. (e) Adding the $\mL_t$ to $\mL_d$ better preserves the topology of the target space.
}
\label{fig_vis}
\end{figure*}

\citet{birdal2021intrinsic} suggests to estimate the intrinsic dimension as the slope between $\log E (\bZ_{n})$ and $ \log n$.
%, where $\bZ_{n}$ is the set of $n$ samples from $\bZ$ . 
Note, the definition of intrinsic dimension used in \citet{birdal2021intrinsic}
%\hl{is different} 
is based on the $0$-dimensional persistent homology $\text{PH}_{0}(\text{VR}(\bZ_n))$, which is different from ours (Definition \ref{deftn:ID}, coming from~\citet{ghosh2023local}). % (also ours).
% \hl{However, they are both defined for the intrinsic dimension}, 
However, both definitions define the same object, \ie the intrinsic dimension, and it is thus reasonable to exploit \citet{birdal2021intrinsic}'s method to constrain the intrinsic dimension.

Let $\be' = [\log E(\bZ_{n_1}), \log E(\bZ_{n_2}), \cdots, \log E(\bZ_{n_m})]$ , where $\bZ_{n_i}$ is the subset sampled from a batch, with  size $n_i = |\bZ_{n_i}|$. Let $n_i < n_j$ for $i<j$, and $\bn = [\log n_1, \log n_2, \cdots, \log n_m]$. \citet{birdal2021intrinsic} encourage a lower intrinsic dimension feature space by minimizing the slope between $\be'$ and $\bn$, which can be estimated via the least square method:
\begin{equation}
\label{equ_l_dPrime}
    \mL'_d = (m\sum_{i=1}^{m} \bn_i \be'_i - \sum_{i=1}^{m} \bn_i \sum_{i=1}^{m} \be'_i) / (m \sum_{i=1}^{m} \bn_i^2 - (\sum_{i=1}^{m} \bn_i)^2).
\end{equation}
Intuitively, the growth rate of $E(\bZ_{n})$ is proportional to the volume of the corresponding manifold; this volume is proportional to the intrinsic dimension. In fact, there is a classical result on the growth rate of \cite{steele1988growth}, showing that the growth rate (i.e. the slop) can constrain the intrinsic dimension. 

$\mL'_d$ purely encourage the feature space to have a lower intrinsic dimension; sometimes it may even result in an intrinsic dimension lower than that of the target space (see Figure \ref{fig_syn}, Swiss Roll, where the target space is two-dimensional and the feature space is almost one-dimensional.).
In contrast, we wish to lower the ID of the feature space while preventing it from being lower than that of the target space.
We propose to minimize slope between $\be = [\be_1, \be_2, \cdots, \be_m]$ and $\bn$:
\begin{equation}
\label{equ_l_d}
    \mL_d = |(m\sum_{i=1}^{m} \bn_i \be_i - \sum_{i=1}^{m} \bn_i \sum_{i=1}^{m} \be_i) / (m \sum_{i=1}^{m} \bn_i^2 - (\sum_{i=1}^{m} \bn_i)^2)|,
\end{equation}
where $\be_i = \log E(\bZ_{n_i})/\log E(\bY_{n_i})$. 
%By contrast $\be'_i = \log E(\bZ_{n_i})$ in $\mL'_d$.
Compared with $\mL'_d$, $\mL_d$ further exploits the topological information of the target space through $\log E(\bY_{n_i})$. When the feature and target spaces have the same ID, $E(\bZ_{n_i}) = E(\bY_{n_i})$ for all $i$ and $\mL_d=0$ is in its minimal.
As shown in Figure \ref{fig_ld} and Figure \ref{fig_syn}, $\mL_d$ well controls the ID of the feature space while better preserving the topology of the target space.

The topology autoencoder \cite{moor2020topological} enforces the topological similarity between the feature and the target spaces by preserving $0$-dimensional topologically relevant distances from the two spaces. We exploit it as the topology part $\mL_t$:
\begin{align}
\label{equ_l_t}
    \mL_{t}  = &||\bA^{\bZ_{n_m}}[\pi^{\bZ_{n_m}}] - \bA^{\bY_{n_m}}[\pi^{\bZ_{n_m}}]||_2^2 \\
    & + ||\bA^{\bZ_{n_m}}[\pi^{\bY_{n_m}}] - \bA^{\bY_{n_m}}[\pi^{\bY_{n_m}}]||_2^2
    % \mL_{T} = (\sum \bA^{\bZ_n}[\pi^{\bY_n}] - \sum \bA^{\bZ_n}[\pi^{\bZ_n}]) / \sum \bA^{\bZ_n}[\pi^{\bY_n}]
\end{align}
As shown in Figure \ref{fig_lt} and Figure \ref{fig_syn},
%\hl{$\mL_t$ can preserve the topology of the target space, yet it fails to encourage a lower intrinsic dimension.} 
$\mL_t$ well preserves the topology of the target space.
We define the persistent homology regression regularizer, PH-Reg, as $\mL_{R} = \mL_d + \mL_t$.
%, with an intrinsic dimension term $\mL_D$ and a topology term $\mL_T$. 
As shown in Figure \ref{fig_lr} and Figure \ref{fig_syn}, PH-Reg can both encourage a lower intrinsic dimension and preserve the topology of target space.
%We show our regression with PH-Reg (red dotted arrow) in Fig.~\ref{fig:framework}.
The final loss function $\mL_{R}$ is defined as:
\begin{equation}
    \mL_{R} = \mL_m + \lambda_t \mL_t + \lambda_d \mL_d,
\end{equation}
where $\mL_m$ is the task-specific regression loss and $\lambda_d, \lambda_t$ are trade-off parameters, and their values are determined by the value of the task task-specific loss $\mL_m$, \eg for a high $\mL_m$, $\lambda_d$ and $\lambda_t$ should also be set to high values.

\section{Experiments}
\label{sec:experiment}

We compare our method with four methods. 1) Information Dropout (InfDrop) \cite{achille2018information}. InfDrop serves as an IB baseline.  It functions as a regularizer designed based on IB, aiming to learn representations that are minimal, sufficient, and disentangled. 2) Ordinal Entropy (OE) \cite{zhang2023improving}. OE acts as a regression baseline. It takes advantage of classification by learning higher entropy feature space for regression tasks. 3) Birdal's regularizer (\ie, $\mL_d'$) \cite{birdal2021intrinsic}  serves as an intrinsic dimension baseline. 4) Topology Autoencoder (\ie, $\mL_t$) \cite{moor2020topological} serves as a topology baseline. Note that the proposed PH-Reg is mainly introduced to verify the established connections, and we do not aim for state-of-the-art results.

\subsection{Coordinate Prediction on the Synthetic Dataset}
\label{subsection:synthetic}
To verify the topological relationship between the feature space and target space, we synthesize a dataset that contains points sampled from topologically different objects, including Swiss roll, torus, circle and the more complex object ``mammoth'' \citep{Coenen19}. We randomly sample $3000$ points with coordinate $\by \in \mmR^3$ from each object. These $3000$ points are then divided into sets of $100$ for training, $100$ for validation, and $2800$ for testing. Each point $\by_i$ is encoded into a $100$ dimensional vector $\bx_i = [f_1(\by_i), f_2(\by_i), f_3(\by_i), f_4(\by_i), \text{noise}]$, where the dimensions $1$-$4$ are signal and the rest $96$ dimensions are noise. The coordinate prediction task aims to learn the mapping $G(\bx) = \hat{\by}$ from $\bx$ to $\by$, and the mean-squared error $\mL_{\text{mse}}=\frac{1}{N}\sum_i||\hat{\by_i}-\by_i||_2^2$ is adopted as the evaluation metric. We use a two-layer fully connected neural network with 100 hidden units as the baseline architecture. More details are given in Appendix \ref{appendix:syntheticDataset}.

\begin{table}[t!]
	\caption{Results ($\mL_{\text{mse}}$) on the synthetic dataset. We report results as mean $\pm$ standard variance over $10$ runs. \textbf{Bold} numbers indicate the best performance.}
	\label{tab:synthetic}
	\centering
		\scalebox{0.8}{\begin{tabular}{c|cccc}
			\hline
			\multirow{1}[0]{*}{Method}
			& Swiss Roll & Mammoth & Torus & Circle  \\
			\hline
% 			Eigen et al. \citep{eigen2014depth} & 0.769 & 0.158 & 0.641 & -  \\
			Baseline  & 2.99 $\pm$ 0.43 & 211 $\pm$ 55 & 3.01 $\pm$ 0.11 & 0.154 $\pm$ 0.006  \\
			$+$ InfDrop  & 4.15 $\pm$ 0.37 & 367 $\pm$ 50 & 2.05 $\pm$ 0.04& 0.093 $\pm$ 0.003  \\
			$+$ OE  & 2.95 $\pm$ 0.69 & 187 $\pm$ 88 & 2.83 $\pm$ 0.07 & 0.114 $\pm$ 0.007  \\
      \hline
			 $+ \mL'_d$ & 2.74 $\pm$ 0.85 & 141 $\pm$ 104 & 1.13 $\pm$ 0.06 & 0.171 $\pm$ 0.04 \\
			 $+ \mL_d$ & 0.66 $\pm$ 0.08   & 89 $\pm$ 66 & 0.62 $\pm$ 0.12& 0.090 $\pm$ 0.019  \\
			 $+ \mL_t$ & 1.83 $\pm$ 0.70 & 80 $\pm$ 61& 0.95 $\pm$ 0.05& 0.036 $\pm$ 0.004  \\
			 $+ \mL_d+ \mL_t$ & \textbf{0.61} $\pm$ \textbf{0.17} & \textbf{49} $\pm$ \textbf{27} & \textbf{0.61} $\pm$ \textbf{0.05}& \textbf{0.013} $\pm$ \textbf{0.008}   \\
			\hline
		\end{tabular}}
\end{table}
Table \ref{tab:synthetic} shows that encouraging a lower intrinsic dimension while considering the target space ($+\mL_d$) enhances performance, particularly for Swiss Roll and Torus. In contrast, naively lowering the intrinsic dimension ($+\mL'_d$) performs poorly.
%and even worse than the baseline, \ie Torus. 
Enforcing the topology similarity between the feature space and target space ($+\mL'_t$) decreases the $\mL_{\text{mse}}$ by more than $60\%$, except for the Swiss roll. The best gains, however, are achieved by incorporating both $\mL_t$ and $\mL_d$, which decrease the $\mL_{\text{mse}}$ even more than $90\%$ for the circle coordinate prediction task. Figure \ref{fig_syn} shows feature space visualization results using t-SNE ($100$ dimensions $\rightarrow 3$ dimensions). The feature space of the regression baseline shows a similar structure to the target space, especially for Swiss roll and mammoth, which indicates regression potentially captures the topology of the target space. Regression $+ \mL_t$  significantly preserves the topology of the target space. Regression $+ \mL_d$ potentially preserves the topology of the target space, \eg circle, while it primarily reduces the complexity of the feature space by maintaining the same intrinsic dimension as the target space. 
%Regression $+ \mL_d + \mL_t$ preserves the topology information and also simplifies the feature space, \ie reduces its complexity.
Combining both $\mL_d$ and $\mL_t$ in regression preserves the topology information while also reducing the complexity of the feature space, \ie lower its intrinsic dimension.

\begin{figure*}[!t]
\centering
\includegraphics[width=\linewidth]{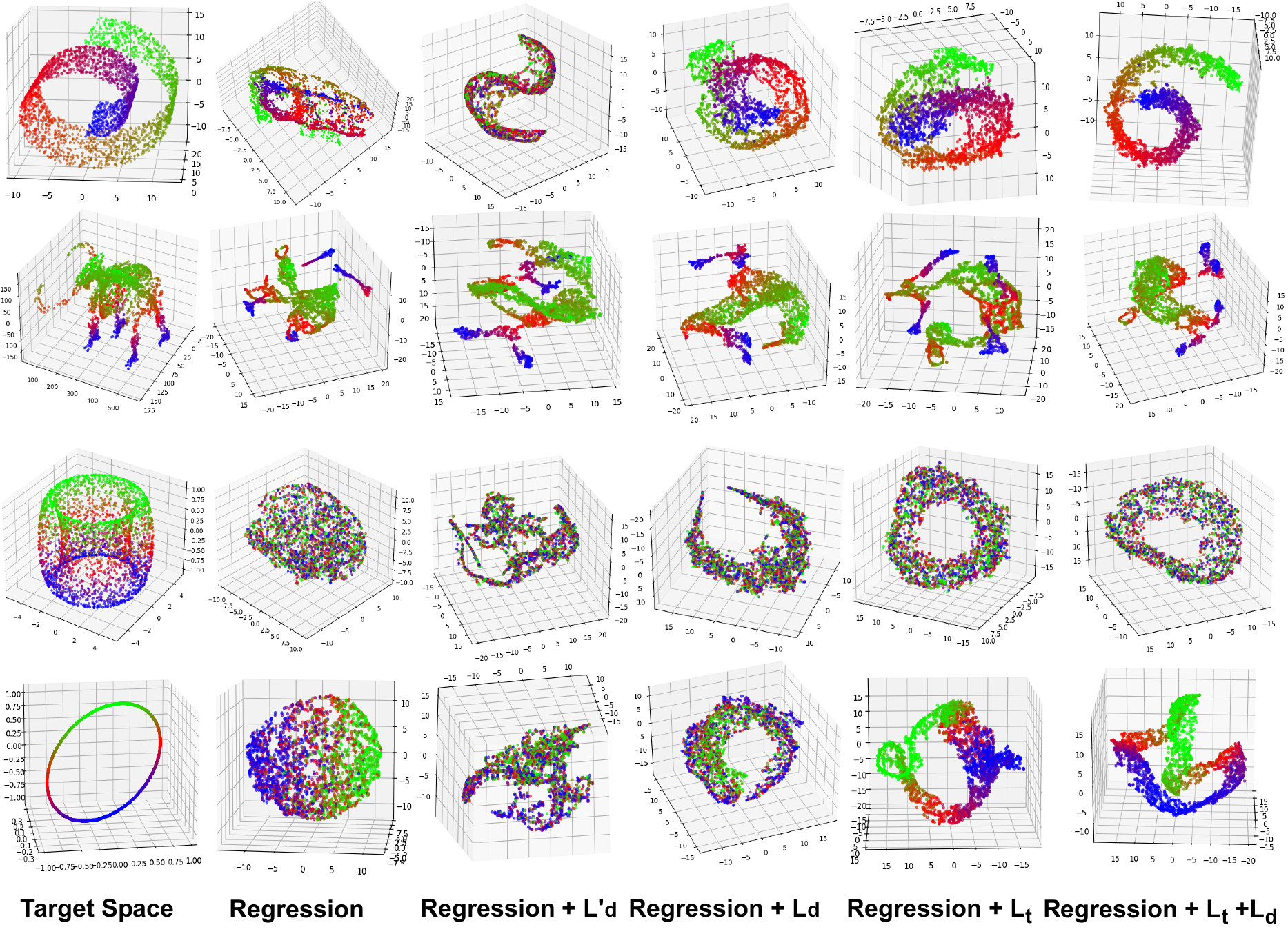}
\caption{t-sne visualization of the feature spaces ($100$ dimensions $\rightarrow 3$ dimensions) with topological different target spaces.}
\label{fig_syn}
\end{figure*}

\subsection{Real-World Regression Tasks}

% We conduct experiments on three real-world regression tasks, including depth estimation (Table \ref{tab:age}), super-resolution (Table \ref{tab:superresolution}) and age estimation (Table \ref{tab:nyu}). The target spaces of the three tasks are topologically different, \ie a $1$- dimensional line for depth estimation, $3$-dimensional space for super-resolution and discrete points for age estimation.  Results on the three tasks demonstrate that both $\mL_t$ and $\mL_d$ can enhance performance, and combining both further boosts the performance. 
% Detailed settings and more discussions are given in Appendix \ref{appendix:realworldTasks}.

We conduct experiments on three real-world regression tasks, including depth estimation (Table \ref{tab:age}), super-resolution (Table \ref{tab:superresolution}) and age estimation (Table \ref{tab:nyu}). The target spaces of the three tasks are topologically different, \ie a $1$- dimensional line for depth estimation, $3$-dimensional space for super-resolution and discrete points for age estimation.  
Detailed settings, related introductions and more discussions are given in Appendix \ref{appendix:realworldTasks}.

Results on the three tasks demonstrate that both $\mL_t$ and $\mL_d$ can enhance performance, and combining both further boosts the performance. Specifically, combining both achieves $0.48$ overall improvements (\ie ALL) on age estimation, a PSNR improvement of $0.096$ on super-resolution for Urban100, and a reduction of $6.7\%$ $\delta_1$ error on depth estimation.

\begin{table}[t!]
	\caption{Quantitative comparison (MAE) on AgeDB. We report results as mean $\pm$ standard variance over $3$ runs.
	%Baseline and Baseline$^*$ denote the results from the original paper and our re-trained model. 
% 	\AY{\hl{ALL, Many, etc. refer to XX}}
	\textbf{Bold} numbers indicate the best performance.
	}
 	\label{tab:age}
	\centering
		\scalebox{0.8}{\begin{tabular}{c|ccccc}
			\hline
			\multirow{1}[0]{*}{Method}
			& ALL & Many & Med. & Few  \\
			\hline
			Baseline & 7.80 $\pm$ 0.12 & 6.80 $\pm$ 0.06 & 9.11	$\pm$ 0.31 & 13.63	$\pm$ 0.43  \\
			 $+$ InfDrop & 8.04	$\pm$ 0.14 & 7.14	$\pm$ 0.20 & 9.10	$\pm$ 0.71 & 13.61	$\pm$ 0.32 \\
			 $+$ OE & 7.65	$\pm$ 0.13 & 6.72	$\pm$ 0.09 & 8.77	$\pm$ 0.49 & 13.28	$\pm$ 0.73  \\
% 			Baseline$^*$ & 7.73 & 6.74 & 9.05 & 13.35 & 5.05 & 4.38 & 6.44 & 9.77 \\
   \hline
			$+ \mL'_d$ & 7.75	$\pm$ 0.05 & 6.80	$\pm$ 0.11 & 8.87	$\pm$ 0.05 & 13.61	$\pm$ 0.50 \\
			$+ \mL_d$ & 7.64 $\pm$ 0.07 & 6.82	$\pm$ 0.07 & 8.62	$\pm$ 0.20 & 12.79	$\pm$ 0.65   \\
			$+ \mL_t$ & 7.50	$\pm$ 0.04 & 6.59	$\pm$ 0.03 & 8.75	$\pm$ 0.03 & 12.67	$\pm$ 0.24 \\
			$+ \mL_d + \mL_t$ &  \textbf{7.32}	$\pm$ \textbf{0.09} & \textbf{6.50}	$\pm$ \textbf{0.15} & \textbf{8.38}	$\pm$ \textbf{0.11} & \textbf{12.18}	$\pm$ \textbf{0.38}  \\
   \hline
		\end{tabular}}
\end{table}

\begin{table}[t]
	\caption{Quantitative comparison (PSNR(dB)) of super-resolution results with public benchmark and DIV2K validation set. \textbf{Bold} numbers indicate the best performance.}
	\label{tab:superresolution}
	\centering
		\scalebox{0.8}{\begin{tabular}{c|ccccc}
			\hline
			\multirow{1}[0]{*}{Method}
			& Set5 & Set14 & B100 & Urban100 & DIV2K \\
			\hline
% 			Eigen et al. \citep{eigen2014depth} & 0.769 & 0.158 & 0.641 & -  \\
			Baseline & 32.241 & 28.614 & 27.598 & 26.083 & 28.997  \\
			$+$ InfDrop & 32.219 & 28.626 & 27.594 & 26.059 & 28.980  \\
			$+$ OE & 32.280 & 28.659 & 27.614 & 26.117 & 29.005  \\
      \hline
			 $+ \mL'_d$ & 32.252 & 28.625 & 27.599 & 26.078 & 28.989  \\
			 $+ \mL_d$ & 32.293 & 28.644 & 27.619 & 26.151 & 29.022  \\
			 $+ \mL_t$ & \bf{32.322} & 28.673 & 27.624 & 26.169 & 29.031  \\
			 $+ \mL_d+ \mL_t$ & 32.288 & \bf{28.686} & \bf{27.627} & \bf{26.179} & \bf{29.038}  \\
			\hline
		\end{tabular}}
\end{table}

\begin{table}[t]
	\caption{Depth estimation results with NYU-Depth-v2. \textbf{Bold} and \underline{underline} numbers indicate the best and second best performance, respectively.}
	\label{tab:nyu}
	\centering
		\scalebox{0.8}{\begin{tabular}{l|cccccc}
			\hline
			\multirow{1}[0]{*}{Method}
			& $\delta_1$~$\uparrow$& $\delta_2$~$\uparrow$& $\delta_3$~$\uparrow$ & REL~$\downarrow$ & RMS~$\downarrow$ & $\log_{10}$~$\downarrow$ \\
			\hline
			Baseline & 0.792& 0.955& 0.990  & 0.153 & 0.512 & 0.064\\
			$+$ InfDrop & 0.791& 0.960& 0.992  & 0.153 & 0.507 & 0.064\\   
			$+$ OE & \textbf{0.811}& - & -  & \textbf{0.143} & \textbf{0.478} & \textbf{0.060}\\    
   \hline
			  $+\mL'_d$ & 0.804& 0.954& 0.988 & 0.151 & 0.502 & 0.063\\
			  $+\mL_d$ & 0.795& \textbf{0.959}& \textbf{0.992} & 0.150 & 0.497 & 0.063\\
			  $+\mL_t$ & 0.798& 0.958& 0.990 & 0.149 & 0.502 & 0.063\\
			  $+\mL_d + \mL_t$ & \underline{0.807}& \textbf{0.959}& \textbf{0.992} & \underline{0.144} & \underline{0.481} & \underline{0.061}\\
			\hline
		\end{tabular}}
\end{table}

\subsection{Ablation Studies}

\begin{figure*}[!t]
	\centering
	\subfigure[MSE with $\lambda_t$]{	
	\label{fig_lamda_t_MSE}
	\includegraphics[width=0.241\linewidth,height=0.2\linewidth]{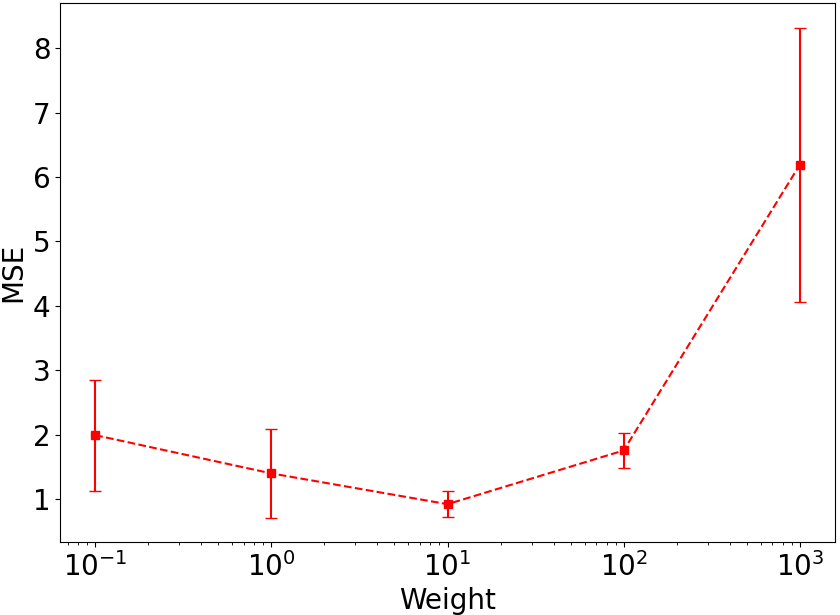}}
	\subfigure[MSE with $\lambda_d$]{	
	\label{fig_lamda_d_MSE}
	\includegraphics[width=0.241\linewidth,height=0.2\linewidth]{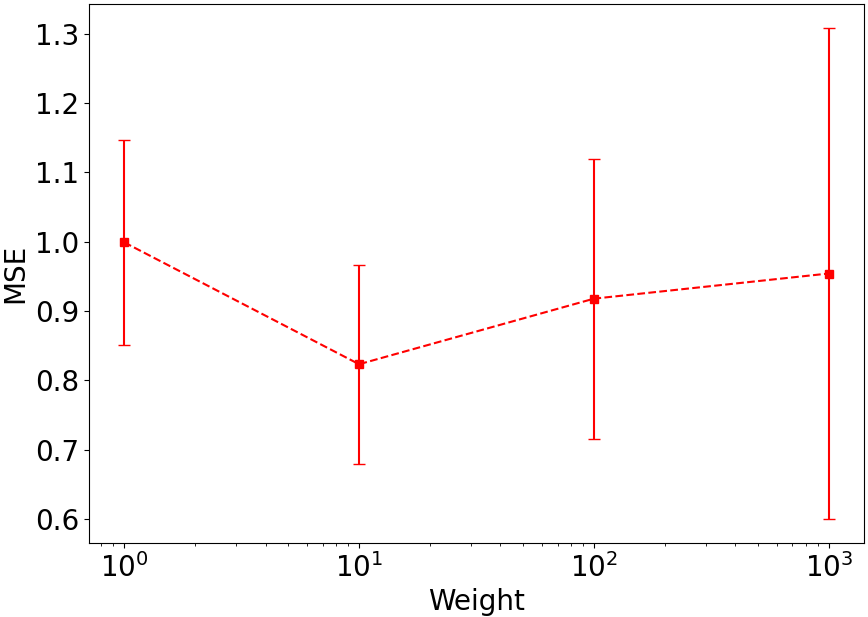}}
	\subfigure[MSE with sample size]{	
	\label{fig_sample_MSE}
	\includegraphics[width=0.241\linewidth,height=0.2\linewidth]{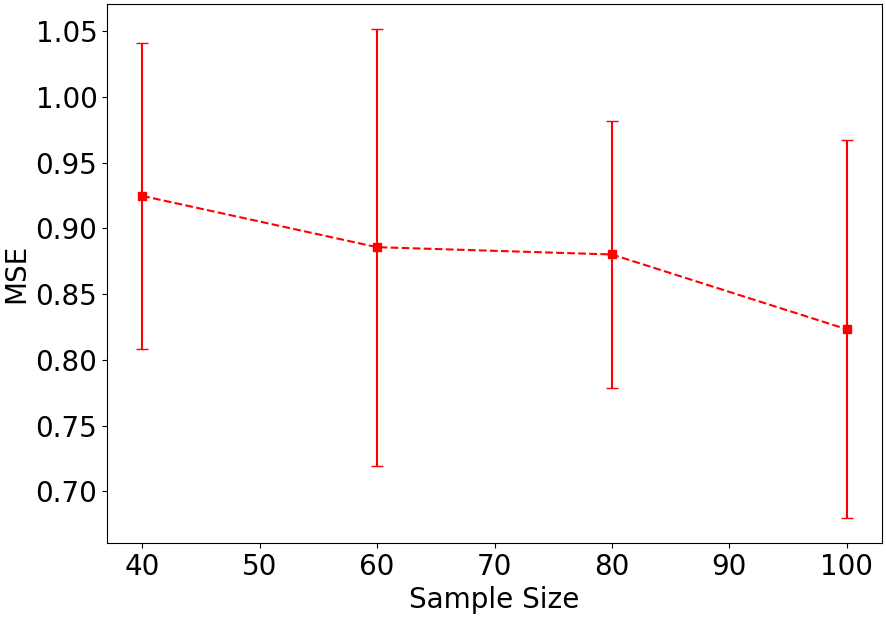}}
         \subfigure[ID of different methods]{
 	\label{fig_ablation_ID}
	\includegraphics[width=0.241\linewidth,height=0.2\linewidth]{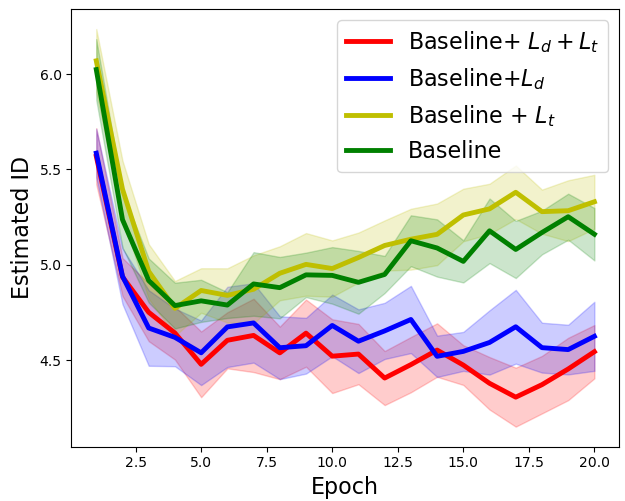}}
	\caption{Ablation study based on (a-c) the Swiss roll coordinate prediction task and (d) the depth estimation task. }
	\label{fig:ablation}
\end{figure*}

\textbf{Hyperparameter $\lambda_t$ and $\lambda_d$}: We maintain $\lambda_d$ and $\lambda_t$ at their default value $10$ for Swiss roll coordinate prediction, and we vary one of them to examine their impact. Figure \ref{fig_lamda_t_MSE} shows when $\lambda_t \leq 10$, the MSE decreases consistently as $\lambda_t$ increases. However, it tends to overtake the original learning objective when set too high, \ie $1000$. Regarding the  $\lambda_d$, as shown in Figure \ref{fig_lamda_d_MSE}, MSE remains relatively stable over a large range of $\lambda_d$, with a slight increase in variance when $\lambda_d=1000$.

\textbf{Sample Size ($n_m$)}: 
In practice, we model the feature space using a limited number of samples within a batch. 
%For dense prediction tasks like depth prediction, the available number of samples is more than required 
For dense prediction tasks, 
the available No. of samples is very large
(No. pixels per image $\times$ batch size), 
while it is constrained to the batch size for image-wise prediction tasks.
%like age estimation, the number of available samples is constrained to the batch size. 
We investigate the influence of $n_m$ from Eq. \ref{equ_l_d} and \ref{equ_l_t} on Swiss roll coordinate prediction. 
%We set the minimum sample size ($n_1$) to $20$, and vary the maximum sample size $n_m$ from $40$ to $100$. 
Figure \ref{fig_sample_MSE} shows our PH-Reg performs better with a larger $n_m$, while maintaining stability even with a small $n_m$.

\textbf{ID of different methods}: Figure \ref{fig_ablation_ID} displays the intrinsic dimension of the last hidden layer, estimated using TwoNN \citep{facco2017estimating}, for the testing set of NYU-Depth-v2 from different methods throughout training. While our method is based on Birdal's estimator \citep{birdal2021intrinsic}, another estimator, TwoNN, captures a decrease in ID when applied $\mL_d$. 
We observe that without $\mL_d$, the intrinsic dimension tends to increase after epoch $3$, potentially overfitting details, whereas $\mL_d$ prevents such a trend. 

\textbf{Efficiency}: 
Efficiency-wise, the computing complexity equals finding the minimum spanning tree from the distance matrix of the samples, which have a complexity of $\mO(n_m^2 \log n_m)$ using the simple Kruskal's Algorithm, and it can speed up with some advanced methods \citep{bauer2021ripser}.  
The synthetic experiments (Table \ref{tab:efficiency}) use a simple 2-layer MLP, so the regularizer adds significant computing time. However, the real-world experiments on depth estimation (Table \ref{tab:efficiency}) use a ResNet-50 backbone, and the added time and memory are negligible (18.6\% and 0.3\%, respectively), even with $n_m=300$. 
%Note that 
These increases are only during training and do not add 
%additional 
demands for inference.

\begin{table}[t!]
	\caption{Quantitative comparison of the time consumption and memory usage on the synthetic dataset and NYU-Depth-v2, and the corresponding training times are $10000$ and $1$ epoch, respectively. 
	}
	\centering
 \label{tab:efficiency}
		\scalebox{0.7}{\begin{tabular}{cc|cc|cc}
			\hline
			\multirow{3}[0]{*}{$n_m$} &\multirow{3}[0]{*}{Regularizer}& \multicolumn{2}{c|}{Coordinate Prediction  } & \multicolumn{2}{c}{Depth Estimation }  \\
			& & \multicolumn{2}{c|}{ ($2$ Layer MLP) } & \multicolumn{2}{c}{(ResNet-50)}  \\
 			\cline{3-6}
			& & Training(s) & Memory (MB) & Training(s) & Memory (MB) \\
			\hline
			0 & - & 8.88 & 959 & 1929 & 11821   \\
			100 & $\mL_t$ & 175.06 & 959 & 1942 & 11833   \\
			100 & $\mL_d$ & 439.68 & 973 & 1950 & 12211   \\
			100 & $\mL_t + \mL_d$ & 617.41 & 973 & 1980 & 12211   \\
			300 & $\mL_t + \mL_d$ & - & - & 2370 & 12211   \\
   			\hline
		\end{tabular}}
\end{table}

\section{Conclusion}

In this paper, we establish novel connections between topology and the IB principle for regression representation learning. The established connections imply that the desired $\bZ$ should exhibit topological similarity to the target space and share the same intrinsic dimension as the target space. Inspired by the connections, we proposed a regularizer 
to learn the desired $\bZ$.
%named PH-Reg, to lower the intrinsic dimension of feature space and keep the topology of the target space for regression. 
Experiments on synthetic and real-world regression tasks demonstrate its benefits. 

\section*{Acknowledgement}
This research / project is supported by the Ministry of Education, Singapore, under the Academic Research Fund Tier 1 (FY2022).

\section*{Impact Statement}
This paper presents work whose goal is to advance the field of Machine Learning. There are many potential societal consequences of our work, none of which we feel must be specifically highlighted here.

\bibliography{main}
\bibliographystyle{icml2024}

\appendix

\newpage
\onecolumn

\section{Proofs}
\label{appendix:proofs}

\subsection{Proof of the Theorem \ref{thm:bottleneck}}
\label{appendix:thm:bottleneck}

\textbf{Theorem \ref{thm:bottleneck} }
\emph{Assume that the conditional entropy $\mH(\bZ|\bX)$ is a fixed constant for $\bZ \in \mathcal Z$ for some set  $\mathcal Z$ of the random variables, or $\bZ$ is determined given $\bX$. Then,  $\min_{\bZ} ~ \mI\mB= \min_{\bZ} ~ \{(1-\beta) \mH(\bY|\bZ) + \beta \mH(\bZ|\bY) \}$.}

\begin{proof}
From the definition of the mutual information, we have
$$
\mI(\bZ; \bX) = \mH(\bZ) - \mH(\bZ|\bX) = \mI(\bZ; \bY) + \mH(\bZ|\bY) - \mH(\bZ|\bX).
$$
By substituting the right-hand side of this equation into $\mI(\bZ; \bX)$, 
\begin{align}
\mI\mB = - \mI(\bZ; \bY) + \beta \mI(\bZ; \bX) = (\beta-1)\mI(\bZ; \bY) + \beta \mH(\bZ|\bY) - \beta \mH(\bZ|\bX)
\end{align}
Since $\mI(\bZ; \bY)=\mH(\bY) - \mH(\bY|\bZ)$,
\begin{align}
\mI\mB &= (\beta-1)(\mH(\bY) - \mH(\bY|\bZ)) + \beta \mH(\bZ|\bY) - \beta \mH(\bZ|\bX)
\\ & = (1-\beta)\mH(\bY|\bZ)  + \beta \mH(\bZ|\bY) + (\beta-1) \mH(\bY) - \beta \mH(\bZ|\bX).
\end{align}

1) If $\mH(\bZ|\bX)$ is a constant for $\bZ \in \mathcal Z$. Since $\mH(\bY)$ is a fixed constant for any $\bZ$, this implies that
$$
\mI\mB=(1-\beta) \mH(\bY|\bZ) + \beta \mH(\bZ|\bY) + C,
$$
where $C$ is a fixed constant for $\bZ \in \mathcal Z$. Thus:
$$
    \min_{\bZ} ~ \mI\mB= \min_{\bZ} ~ \{(1-\beta) \mH(\bY|\bZ) + \beta \mH(\bZ|\bY) \}.
$$

2) If $\bZ$ is determined given $\bX$, then $\mH(\bZ|\bX)$ is not a term can be optimized. Since $\mH(\bY)$ is a fixed constant for any $\bZ$: 
$$
    \min_{\bZ} ~ \mI\mB= \min_{\bZ} ~ \{(1-\beta) \mH(\bY|\bZ) + \beta \mH(\bZ|\bY) \}.
$$
\end{proof}

\subsection{Proof of the Theorem \ref{thm:bound}  and Proposition \ref{prop_supportTheorem1}}
\label{appendix:thm:bound}

\textbf{Theorem \ref{thm:bound} }
\emph{Consider dataset $S = \{\bx_i,\bz_i, \by_i\}^N_{i=1}$ sampled from distribution $P$, where $\bx_i$ is the input, $\bz_i$ is the corresponding representation, and $\by_i$ is the label. Let  $d_{max} = \max_{\by\in \mY} \min_{\by_i \in S} ||\by-\by_i||_2$  be the maximum distance of $\by$ to its nearest $\by_i$. Assume  $(\bZ|\bY=\by_i)$ follows a distribution $\mD$ and 
%\hl{the following} holds: 
the dispersion of $\mD$ is bounded by its entropy:
\begin{equation}
    \mmE_{\bz \sim \mD}[||\bz - \bar{\bz}||_2] \leq Q(\mH(\mD)),
\end{equation}
where  $\bar{\bz}$ is the mean of the distribution $\mD$ and $Q(\mH(\mD))$ is some function of 
$\mH(\mD)$. 
%The above implies the dispersion of the distribution $\mD$ is bounded by its entropy, which usually is the case, like the multivariate normal distribution and the uniform distribution. 
Assume the regressor f is $L_1$-Lipschitz continuous, then as $d_{max} \rightarrow 0$, we have:
%~\AY{XX}
\begin{align} \label{eq:1}
    &  \mmE_{\{\bx, \bz, \by\} \sim P} [||f(\bz) - \by||_2]  \leq  \mmE_{\{\bx, \bz, \by\} \sim S}(||f(\bz) -\by||_2) 
    + 2L_1 Q(\mH(\bZ|\bY))
\end{align}
}

\begin{proof}
For any sample $\{\bx_i, \bz_i, \by_i\}$, we define its local neighborhood set $N_i$ as 
\begin{equation}
    N_i = \{\{\bx, \bz, \by\}: ||\by - \by_i||_2< ||\by - \by_j||_2, j \neq i, p(\by)>0 \}.
\end{equation}

For each set $N_i$, we have 
% \begin{equation}
%     \mmE_{(z,y)}[f(z) - y] = \mmE_{N_i} .
% \end{equation}
\begin{align}
    &\mmE_{\{\bx, \bz, \by\} \sim N_i}[||f(\bz) - \by||_2] = \mmE_{\{\bx, \bz, \by\} \sim N_i}[||f(\bz) - f(\bz_i) + f(\bz_i) -\by_i + \by_i - \by||_2] \\
    \leq&\mmE_{\{\bx, \bz, \by\} \sim N_i} [||f(\bz) - f(\bz_i)||_2]
    + \mmE_{\{\bx, \bz, \by\} \sim N_i} [||f(\bz_i) -\by_i||_2] 
    + \mmE_{\{\bx, \bz, \by\}\sim N_i} [||\by_i - \by||_2] \\
    \leq & L_1  \mmE_{\{\bx, \bz, \by\} \sim N_i}[||\bz - \bz_i||_2] + \mmE_{\{\bx, \bz, \by\} \sim N_i} [||f(\bz_i) -\by_i||_2] + d_{max}\\
    = & L_1  \mmE_{\{\bx, \bz, \by\} \sim N_i}[||\bz -\bar{\bz}_i + \bar{\bz}_i- \bz_i||_2] + \mmE_{\{\bx, \bz, \by\} \sim N_i} [||f(\bz_i) -\by_i||_2] + d_{max} \\
    \leq & L_1\mmE_{\{\bx, \bz, \by\} \sim N_i}[||\bz -\bar{\bz}_i||_2 + ||\bar{\bz}_i- \bz_i||_2] + \mmE_{\{\bx, \bz, \by\} \sim N_i} [||f(\bz_i) -\by_i||_2] + d_{max} \\
    = &L_1\mmE_{\{\bx, \bz, \by\} \sim N_i}[||\bz -\bar{\bz}_i||_2] + L_1||\bar{\bz}_i- \bz_i||_2 + \mmE_{\{\bx, \bz, \by\} \sim N_i} [||f(\bz_i) -\by_i||_2] + d_{max}
\end{align}

We denote the probability distribution over $\{N_i\}$ as $P'$, where $P(N_i) = P ( \{\bx, \bz, \by\} \in N_i\})$. Then, we have 
\begin{align}
    &\mmE_{\{\bx, \bz, \by\} \sim P}[||f(\bz) - \by||_2] = \mmE_{N_i \sim P'}\mmE_{\{\bx, \bz, \by\} \sim N_i}[||f(\bz) - \by||_2] \\
    \leq & \mmE_{N_i \sim P'} [L_1\mmE_{\{\bx, \bz, \by\} \sim N_i}[||\bz -\bar{\bz}_i||_2] + L_1||\bar{\bz}_i- \bz_i||_2 + \mmE_{\{\bx, \bz, \by\} \sim N_i} [||f(\bz_i) -\by_i||_2] + d_{max}] \\
    = & L_1 \mmE_{N_i \sim P'}\mmE_{\{\bx, \bz, \by\} \sim N_i}[||\bz - \bar{\bz}_i||_2] + L_1\mmE_{N_i \sim P'} ||\bar{\bz}_i -\bz_i||_2 +\mmE_{\{\bx, \bz, \by\} \sim S}(||f(\bz_i) -\by_i||_2) + d_{max} 
\end{align}

As $d_{max} \rightarrow 0$, we can approximate $ \mmE_{N_i \sim P'}\mmE_{\{\bx, \bz, \by\} \sim N_i}[||\bz - \bar{\bz}_i||_2]$ as $ \mmE_{\by_i \sim \mY}\mmE_{\{(\bx, \bz, \by)| \by=\by_i\}}[||\bz - \bar{\bz}_i||_2]$.
%and $\mmE_{N_i \sim P'} ||\bz_i -\bar{\bz}_i||_2$ can also be approximate as $\mmE_{y \sim \mY}[||\bz_i -\bar{\bz}_i||_2]$.
Since $(\bZ|\bY=\by_i) \sim \mD$, we have $\mH(\bZ|\bY) = \mmE_{y \sim \mY} \mH(\bZ| \bY=y) = \mH(\bZ|\bY=\by_i) = \mH(\bZ|\bY=\by_j) =\mH(\mD)$ for all $1\leq i, j \leq N$,  and $\mmE_{N_i \sim P'} ||\bz_i -\bar{\bz}_i||_2$ can thus be approximate as $\mmE_{\{(\bx, \bz, \by)| \by=\by_i\}}||\bz - \bar{\bz}_i||_2$. 
We have: 

% Without losing generality, we assume the $\bz_i = \bz'_1$, then we have:
% \begin{align}
%     \mmE_{\{(\bx, \bz, y); y=y_i\}}[||\bz - \bz_i||_2] 
%     \leq e^{\frac{1}{2Kd}\hat{\mH}(\bZ| \bY=y_i)}
% \end{align}
\begin{align}
    &\mmE_{\{\bx, \bz, \by\} \sim P}[||f(\bz) - \by||_2]\\
    &\leq L_1 \mmE_{N_i \sim P'}\mmE_{\{\bx, \bz, \by\} \sim N_i}[||\bz - \bar{\bz}_i||_2] + L_1\mmE_{N_i \sim P'} ||\bar{\bz}_i -\bz_i||_2 +\mmE_{\{\bx, \bz, \by\} \sim S}(||f(\bz_i) -\by_i||_2) + d_{max} \\
    &= L_1 \mmE_{y_i \sim \mY}\mmE_{\{(\bx, \bz, y)| y=y_i\}}[||\bz - \bar{\bz}_i||_2] + L_1 \mmE_{\{(\bx, \bz, y)| y=y_i\}}||\bz_i - \bar{\bz}_i||_2 + \mmE_{\{\bx, \bz, \by\} \sim S}(||f(\bz_i) -\by_i||_2) \\
    % & \leq L_1 \mmE_{y_i \sim \mY} \sqrt{\frac{d(e^{2\mH(\bZ|\bY=y_i)})^{\frac{1}{d}}}{2\pi e}}+ L_1  \sqrt{\frac{d(e^{2\mH(\bZ|\bY=y_i)})^{\frac{1}{d}}}{2\pi e}} 
    % +\mmE_{\{\bx, \bz, \by\} \sim S}(||f(\bz_i) -\by_i||_2) \\
    & \leq L_1 \mmE_{y_i \sim \mY} [Q(\mH(\bZ|\bY=\by_i))] + L_1  Q(\mH(\bZ|\bY=\by_i)) 
    +\mmE_{\{\bx, \bz, \by\} \sim S}(||f(\bz_i) -\by_i||_2) \\
     & = 2L_1  Q(\mH(\bZ|\bY))
    +\mmE_{\{\bx, \bz, \by\} \sim S}(||f(\bz_i) -\by_i||_2)    
\end{align}

\end{proof}

\textbf{Proposition \ref{prop_supportTheorem1}}
\emph{
    If $\mD$ is a multivariate normal distribution $\mN(\bar{\bz}, \Sigma = k\bI)$, where $k>0$ is a scalar and $\bar{\bz}$ is the mean of the distribution $\mD$. Then, the function $Q(\mH(\mD))$ in Theorem \ref{thm:bound} can be selected as $Q(\mH(\mD)) = \sqrt{\frac{d(e^{2\mH(\mD)})^{\frac{1}{d}}}{2\pi e}}$, where $d$ is the dimension of $\bz$. If $\mD$ is a uniform distribution, then the $Q(\mH(\mD))$ can be selected as $Q(\mH(\mD)) = \frac{e^{\mH(\mD)}}{\sqrt{12}}$.
}

\begin{proof}
We first consider the case when $\mD \sim \mN(\bar{\bz}, \Sigma = k\bI)$.
Assume $\bZ \sim \mN(\bar{\bz}, \Sigma)$, then $\mH(\bZ)= \frac{1}{2}\log(2\pi e)^n|\Sigma|$:
    \begin{align}
    \mH(\bZ) &= -\int_{\bz}p(\bz)\log(p(\bz))d\bz \\
    & = -\int_{\bz} p(\bz) \log \frac{1}{(\sqrt{2\pi})^d |\Sigma|^{\frac{1}{2}}}e^{\frac{-1}{2}(\bz - \bar{\bz})^{\trsp}\Sigma^{-1}(\bz - \bar{\bz})} d\bz \\
    & = -\int_{\bz} p(\bz) \log \frac{1}{(\sqrt{2\pi})^d |\Sigma|^{\frac{1}{2}}}d\bz 
    -\int_{\bz}p(\bz)\log e^{\frac{1}{2}(\bz - \bar{\bz})^{\trsp}\Sigma^{-1}(\bz - \bar{\bz})} d\bz \\
    &= \frac{1}{2}\log(2\pi)^d|\Sigma| + \frac{\log e}{2} \mmE[\sum_{i,j}(\bz_i - \bar{\bz}_i)(\Sigma^{-1})_{ij}(\bz_j - \bar{\bz}_j)] \\
    &=\frac{1}{2}\log(2\pi)^d|\Sigma| + \frac{\log e}{2} \mmE[\sum_{i,j}(\bz_i - \bar{\bz}_i)(\bz_j - \bar{\bz}_j)(\Sigma^{-1})_{ij}] \\
    &=\frac{1}{2}\log(2\pi)^d|\Sigma| + \frac{\log e}{2} \sum_{i,j}\mmE[(\bz_i - \bar{\bz}_i)(\bz_j - \bar{\bz}_j)](\Sigma^{-1})_{ij} \\
    &=\frac{1}{2}\log(2\pi)^d|\Sigma| + \frac{\log e}{2} \sum_j\sum_i \Sigma_{ji}(\Sigma^{-1})_{ij} \\
    &=\frac{1}{2}\log(2\pi)^d|\Sigma| + \frac{\log e}{2} \sum_j(\Sigma\Sigma^{-1})_{j} \\
    &=\frac{1}{2}\log(2\pi)^d|\Sigma| + \frac{\log e}{2} \sum_j I_{jj} \\
    &=\frac{1}{2}\log(2\pi)^d|\Sigma| + \frac{\log e}{2}\\
    &=\frac{1}{2}\log(2\pi e)^d |\Sigma|
\end{align}

We have the following:
\begin{align}
    \mmE [||\bz - \bar{\bz}||_2^2] = \tr({\Sigma}) = dk.
\end{align}
The following also holds:
\begin{align}
    |\Sigma| = k^d.
\end{align}

Thus, we have:
\begin{align}
    (\mmE [||\bz - \bar{\bz}||_2])^2 \leq \mmE [||\bz - \bar{\bz}||_2^2] = d|\Sigma|^{\frac{1}{d}} = d(\frac{e^{2\mH(\bZ)}}{(2\pi e)^d})^{\frac{1}{d}}= \frac{d(e^{2\mH(\bZ)})^{\frac{1}{d}}}{2\pi e}
\end{align}

Finally, 
    \begin{equation}
        \mmE [||\bz - \bar{\bz}||_2] \leq \sqrt{\frac{d(e^{2\mH(\bZ)})^{\frac{1}{d}}}{2\pi e}}
    \end{equation}
Thus, $Q(\mH(\mD))$ in Theorem \ref{thm:bound} can be selected as $Q(\mH(\mD)) = \sqrt{\frac{d(e^{2\mH(\mD)})^{\frac{1}{d}}}{2\pi e}}$, when $\mD \sim \mN(\bar{\bz}, \Sigma = k\bI)$.

Similarly, if $\mD$ is a uniform distribution $U(a,b)$, then its variance is given by:
\begin{equation}
    \mmE [||\bz - \bar{\bz}||_2^2] = \frac{(b-a)^2}{12},
\end{equation}
and its entropy is given by:
\begin{equation}
    \mH(\bD) = \log (b-a).
\end{equation}

We have: 
\begin{align}
    (\mmE [||\bz - \bar{\bz}||_2])^2 \leq \mmE [||\bz - \bar{\bz}||_2^2] = \frac{(b-a)^2}{12} = \frac{e^{2\mH(\mD)}}{12}
\end{align}

Finally,
    \begin{equation}
        \mmE [||\bz - \bar{\bz}||_2] \leq \frac{e^{\mH(\mD)}}{\sqrt{12}}
    \end{equation}

Thus, $Q(\mH(\mD))$ in Theorem \ref{thm:bound} can be selected as $Q(\mH(\mD)) = \frac{e^{\mH(\mD)}}{\sqrt{12}}$, when $\mD$ is a uniform distribution

\end{proof}

\subsection{Proof of the Theorem \ref{thm:entropyToID}}
\label{appendix:thm:entropyToID}

We first show a straightforward result of [\cite{ghosh2023local}, Proposition 1]:
\begin{lemma}
\label{lemma:entropyToID}
    Assume that $\bz$ lies in a manifold $\mM$ and the $\mM_i \subset \mM$ is a 
%\hl{manifold corresponding to $\mH(\bZ|\bY=\by_i)$}. 
manifold corresponding to the distribution $(\bz|\by=\by_i)$.
Assume for all features $\bz_i \in \mM_i$, the following holds:
\begin{equation}
\int_{||\bz-\bz_i|| \leq \epsilon} P(\bz) d\bz = C(\epsilon),
\end{equation}
where $C(\epsilon)$ is some function of $\epsilon$.
The above imposes a constraint where the distribution $(\bz|\by=\by_i)$ is uniformly distributed across $\mM_i$.
%~\AY{add sentence here explaining eq. 2 in lay-man terms.} 
Then, as $\epsilon \rightarrow 0 ^{+}$, 
%\ie
% ~\AY{as the distance between XX} 
%the distribution $(\bz|\by=\by_i)$ is uniformly distributed, 
we have:
\begin{equation}
    \mH(\bZ|\bY) = \mmE_{\by_i \sim \mY} \mH(\bZ| \bY=\by_i) = \mmE_{\by_i \sim \mY}[-\log(\epsilon)\text{Dim}_{\text{ID}} \mM_i + \log\frac{K}{C(\epsilon)}],
\end{equation}
for some fixed scalar K. $\text{Dim}_{\text{ID}}\mM_i$ is the intrinsic dimension of the manifold $\mM_i$.
\end{lemma}
\begin{proof}
By using the same proof technique as [\cite{ghosh2023local}, Proposition 1], we can show
\begin{equation}
    \mH(\bZ| \bY=\by_i) =-\log(\epsilon)\text{Dim}_{\text{ID}} \mM_i + \log\frac{K}{C(\epsilon)},
\end{equation}
Since $\mH(\bZ|\bY) = \mmE_{\by_i \sim \mY} \mH(\bZ| \bY=\by_i)$, the result follows.
\end{proof}

\textbf{Theorem \ref{thm:entropyToID}}
\emph{
Assume that $\bz$ lies in a manifold $\mM$ and the $\mM_i \subset \mM$ is a 
%\hl{manifold corresponding to $\mH(\bZ|\bY=\by_i)$}. 
manifold corresponding to the distribution $(\bz|\by=\by_i)$.
Let $C(\epsilon)$ be some function of $\epsilon$:
\begin{equation}
C(\epsilon) = \int_{||\bz-\bz'|| \leq \epsilon} P'(\bz) d\bz,
\end{equation}
where $P'(\bz)$ is the probability of $\bz$ when $(\bz|\by=\by_i)$ is uniformly distributed across $\mM_i$, and $\bz'$ is any point on $\mM_i$.
Then, as $\epsilon \rightarrow 0 ^{+}$, 
we have:
\begin{align}
    \mH(\bZ|\bY) &= \mmE_{\by_i \sim \mY} \mH(\bZ| \bY=\by_i) \leq \mmE_{\by_i \sim \mY}[-\log(\epsilon)\text{Dim}_{\text{ID}} \mM_i + \log\frac{K}{C(\epsilon)}],
\end{align}
for some fixed scalar K. $\text{Dim}_{\text{ID}}\mM_i$ is the intrinsic dimension of the manifold $\mM_i$.}

\begin{proof}
Since the uniform distribution has the largest entropy over all distributions over the support $\mM_i$, based on  Lemma \ref{lemma:entropyToID}, we thus have:
\begin{align}
    \mH(\bZ|\bY) = \mmE_{\by_i \sim \mY} \mH(\bZ| \bY=\by_i) \leq \mmE_{\by_i \sim \mY}[-\log(\epsilon)\text{Dim}_{\text{ID}} \mM_i + \log\frac{K}{C(\epsilon)}],
\end{align}

\end{proof}

\subsection{Proof of the Proposition \ref{prop:optimal}}
\label{appendix:prop:optimal}

\textbf{Proposition \ref{prop:optimal}}
\emph{
Let the target $\bY =\bY' + \bN'$  
where $\bY'$ is fully determined by $\bX$ and $\bN'$ is 
the aleatoric uncertainty that is independent of $\bX$. Assume the underlying mapping $f'$ from $\bZ$ to $\bY'$ and $g'$ from $\bY'$ to $\bZ$ are continuous, where the continuous mapping is based on the topology induced by the Euclidean distance. Then the representation $\bZ$ is optimal if and only if $\bZ$ is homeomorphic to $\bY'$. 
}

\begin{proof}
If $\bZ$ is optimal (optimal $\bZ$ $\Rightarrow$ $\bZ$ is homeomorphic to $\bY'$): 
\begin{align}
    \mH(\bY|\bZ) = \mH(\bY' + \bN' | \bZ) = \mH(\bY'|\bZ) + \mH(\bN'|\bZ,\bY') = \mH(\bY'|\bZ) + \mH(\bN'|\bY'),
\end{align}
\begin{align}
    \mH(\bY|\bX) = \mH(\bY' + \bN' | \bX) = \mH(\bY'|\bX) + \mH(\bN'|\bX,\bY') = \mH(\bY'|\bX) + \mH(\bN'|\bY').
\end{align}

Since $\bZ$ is optimal, we have $\mH(\bY|\bZ) = \mH(\bY|\bX)$. Based on the two equations above, we have:
\begin{align}
    \mH(\bY'|\bZ) = \mH(\bY'|\bX).
\end{align}
Since $\bY'$ is fully determined by $\bX$ and $\mH(\bY'|\bZ) = \mH(\bY'|\bX)$, $\bY$ is also fully determined by $\bZ$. 
Thus, for each $\bz_i \in \bZ$, there exists and only exists one $\by'_i \in \bY'$ corresponding to the $\bz_i$, and thus the mapping function $f'$ exists. 

$\bZ$ is optimal also means $\bZ$ is fully determined given $\bY$, Since $\bN'$ is independent of $\bZ$:
\begin{align}
    \mH(\bZ|\bY) = \mH(\bZ|\bY' +\bN') = \mH(\bZ|\bY'),
\end{align}
thus, for each $\by'_i$, there exist and only exist one $\bz_i$ corresponding to the $\by'_i$. Thus, the mapping function $f'$ is a bijection, and its inverse $f'^{-1}$ is $g'$. Since $f'$ and $f'^{-1}$ are continuous, $\bZ$ is homeomorphic to $\bY$.

If $\bZ$ is homeomorphic to $\bY'$: ( $\bZ$ is homeomorphic to $\bY'$ $\Rightarrow$ optimal $\bZ$ ): 

$\bZ$ is homeomorphic to $\bY'$ means a continuous bijection exist between $\bZ$ and $\bY'$, thus $\mH(\bZ|\bY) = \mH(\bZ|\bY') = \mH(\bY'|\bZ)$ and $\mH(\bZ|\bY)$ is minimal. We have:
\begin{align}
    \mH(\bY|\bZ)= \mH(\bY'|\bZ) + \mH(\bN'|\bY') = \mH(\bN'|\bY') = \mH(\bY'|\bX) + \mH(\bN'|\bY') = \mH(\bY|\bX),
\end{align}
thus, $\bZ$ is optimal.
\end{proof}

\section{Discussions about Theorem \ref{thm:bottleneck}}
\label{appendix:explainTheorem1}
\textbf{Choice of $\beta$ in Theorem \ref{thm:bottleneck}}: $\beta$ in $(0, 1)$ means we need to maximize $\mI(\bZ; \bY)$ for sufficiency, while we want to minimize $\mI(\bZ; \bX)$ for minimality. When $\beta > 1$, then we value the minimality more than the sufficiency, resulting in the need to maximize $\mH(\bY|\bZ)$. But, in the typical setting, we always value $\mI(\bZ; \bY)$ more than $\mI(\bZ, \bX)$ for a good task-specific performance, and $\beta >0$ will lead $\bZ$ compressed to be a single point, as $\mH(\bZ|\bY)$ is minimal and $\mH(\bY|\bZ)$ is maximized in this case. The qualitative behavior should change based on $0 < \beta < 1$ or $\beta > 1$, as it controls which we value the more: sufficiency or minimality.

\textbf{Difference between the target $\bY$ and the predicted $\bY'$}: The target $\bY$ is different from the predicted $\bY'$. $\mH(\bY'|\bZ)$ always equals to $0$ if we exploit neural networks as deterministic functions. In an extreme case, we can treat the predicted $\bY'$ as the representation $\bZ$, which shows minimizing $\mH(\bY'|\bY) = 0$ is the learning target. 
From the invariance representation learning point of view, lowering $\mH(\bZ|\bY)$ is learning invariance representations with respect to $\bY$.

\textbf{More discussions about the assumptions:} For discrete entropy, $\bZ$ is determined given $\bY$ implies $\mH(\bZ|\bY)$ is a constant for $\bZ$, as $\mH(\bZ|\bY) = 0$ here. However, this does not hold for differential entropy. For differential entropy, $\bZ$ is determined given $\bY$ means $\mH(\bZ|\bY) = -\infty$, since given Y, Z is distributed as a delta function in this case.

\section{Connections with the bound in \citet{kawaguchi2023does}}
\label{appendix:compareBound}
\citet{kawaguchi2023does} provide several bounds that are all applicable for both classification and regression for various cases. In the case where the encoder model $\phi$ (whose output is $z$) and the training dataset of the downstream task $\hat S=\{x_i,y_i\}_{i=1}^N$ are dependent (e.g., this is when $z=\phi(x)$ for $x \notin S$ is dependent of all $N$ training data points $(x_i,y_i)_{i=1}^N$ through the training of $\phi$ by using $\hat S$), they show that any valid and general generalization bound of the information bottleneck must include two terms, $I(\bX; \bZ|\bY)$ and $I(\phi; \hat S)$, where the second term measures the effect of overfitting the encoder $\phi$. This is because the encoder $\phi$ can compress all information to minimize $I(\bX; \bZ|\bY)$ arbitrarily well while overfitting to the training data: e.g., we can simply set $\phi(x_i)=y_i$ for all $(x_i,y_i) \in \hat S$ and $\phi(x)=c\neq y$ for all $(x,y) \notin \hat S$ for some constant $c$ to achieve the best training loss while minimizing $I(\bX; \bZ|\bY)$ and performing arbitrarily poorly for test loss. Given this observation, they prove the first rigorous generalization bounds for two separate cases based on the dependence of $\phi$ and $\hat S$. Their generalization bounds scales with $I(\bX; \bZ|\bY)$ without the second term  $I(\phi; \hat S)$ in case of $\phi$ and $\hat S$ being independent, and with $I(\bX; \bZ|\bY)$ and $I(\phi; \hat S)$ in case of $\phi$ and $\hat S$ being dependent.

In Theorem \ref{thm:bound}, we consider the case where $\phi$ and $\hat S$ are independent, since $z$ in $(x,z,y) \sim P$ is drawn without dependence on the entire $N$ data points $\{x_i,y_i\}_{i=1}^N$ in equation \eqref{eq:1}. Thus, our results are consistent and do not contradict previous findings. Unlike the previous bounds, our bound is determined by the function $Q$, which characterizes the dispersion or the standard deviation of a distribution by its entropy. The function $Q$ exists for general cases, as the dispersion or the standard deviation and the entropy commonly can be estimated for a specific distribution. We thus can find a function $Q$ to upper bound the relationship on the entropy and its dispersion or the standard deviation.

It is worth mentioning that we are not targeting a tight or an advanced bound here. Our bound is introduced to support the analysis that follows after Theorem \ref{thm:bound}, which is challenging with the previous bounds. 

\section{From Proposition \ref{prop:optimal}, whether the optimal representation $\bZ$ is the one that equals the ground-truth label?}
\label{appendix:discussionProposition2}

Such $\bZ$ can be one of the optimal/best representations under the negligible aleatoric uncertainty setting. However, $\bZ$ is not unique and Proposition \ref{prop:optimal} is broader than this statement as it says that the optimal representation $\bZ$ should be homeomorphic to the ground truth - i.e. they only need to be topologically equivalent.  In this regard, the feature space may resemble a square, while the target space is a circle. 

A practical benefit of Proposition \ref{prop:optimal} is its guidance on the desirable topological properties of $\bZ$. For example, if the target space is a single connected component ( i.e. $\beta_0 = 1$), then the feature space should also be similar; this does not hold in general (the task-specific loss alone cannot guarantee a single connected feature space and the topology of the feature space is also influenced by the input space. In addition, we observe empirically on depth estimation that the feature space sometimes consists of multiple disconnected components).

\section{Discussions about the regressor}
\label{appendix:discussionRegressor}

\textbf{1. Do the appropriate properties (i.e., lower ID and homeomorphism) depend on the regressor?} 

Although the appropriate feature space may vary from regressor to regressor, the appropriate \textbf{properties}, as supported by our theorems, do not depend on the regressor. Specifically, our Theorem \ref{thm:bottleneck} shows that the IB tradeoff is fully determined by the values of $\mathcal H(\bZ|\bY)$ and $\mathcal H(\bY|\bZ)$.  These values of the entropy terms do not depend on the regressor that maps $\bZ$ to the predicted $\hat{\bY}$. In addition, for depth estimation, representations learned \textbf{purely} by PH-Reg without any other loss terms are also highly competitive (see Table \ref{tab:append_nyu}).

\begin{table}[t]
	\caption{Depth estimation on NYU-Depth-V2. Here, Random represents the encoder is fixed in a random state, while PH-Reg means we first train the encoder purely with PH-Reg for 1 epoch, then fix it and train the regressor.}
	\label{tab:append_nyu}
	\centering
		\scalebox{1}{\begin{tabular}{l|c|cccc}
			\hline
			\multirow{1}[0]{*}{Encoder} & {Regressor}
			& $\delta_1$~$\uparrow$&  REL~$\downarrow$ & RMS~$\downarrow$ & $\log_{10}$~$\downarrow$ \\
			\hline
			Random & Linear & 0.398& 0.390& 1.144  & 0.153 \\
			PH-Reg & Linear & 0.428& 0.391& 1.043  & 0.153 \\
			Random & NonLinear & 0.412& 0.381& 1.121  & 0.149 \\
			PH-Rege & NonLinear & 0.440& 0.374& 1.052  & 0.141 \\
			\hline
		\end{tabular}}
\end{table}

\textbf{2. Topology regularization with “simple” vs highly expressive regressors}

For both regressors, our regularizer leads to representations with a higher signal-to-noise ratio (since $\mathcal H(\bZ|\bY)$ and $\mathcal H(\bZ|\bY)$ are minimized).  This should make it easier for the regressor to estimate the true underlying signal. However, more expressive regressors have a higher capacity to estimate the underlying signal directly, so the room for improvement from the regularization is reduced.  In the extreme case of too much expressiveness, our regularizer may again lead to improvements as it may limit overfitting,  by minimizing noise in the learned representation. 

The topology properties supported by our theorems align with invariant representation learning (i.e. invariance to noise), where invariance serves as a general prior for desirable representation properties \cite{bengio2013representation}.

\section{Preliminaries on Topology}
\label{appendix:Preliminaries}
The simplicial complex is a central object in topological data analysis, and it can be exploited as a tool to model the `shape' of data. Given a set of finite samples $\bS = \{s_i\}$, the simplicial complex $K$ can be seen as a collection of simplices $\sigma = \{s_0, \cdots, s_k\}$ of varying dimensions: vertices $(|\sigma|=1)$, edges$(|\sigma|=2)$, and the higher-dimensional counterparts$(|\sigma|>2)$. 
The faces of a simplex $\sigma = \{s_0, \cdots, s_k\}$ is the simplex spanned by the subset of $\{s_0, \cdots, s_k\}$. The dimension of the simplicial complex $K$ is the largest dimension of its simplices. A simplicial complex can be regarded as a high-dimensional generalization of a graph, and a graph can be seen as a 1-dimensional simplicial complex. 
For each $\bS$, there exist many ways to build simplicial complexes and the Vietoris-Rips Complexes are widely used:

\begin{deftn}
    (Vietoris-Rips Complexes). Given a set of finite samples $\bS$ sampled from the feature space or target space and a threshold  $\alpha \geq 0$, the Vietoris-Rips Complexes VR$_{\alpha}$ is defined as:
    \begin{equation}
        \text{VR}_{\alpha}(\bS) = \{\{ s_0, \cdots, s_k \}, s \in \bS | d(s_i, s_j) \leq \alpha\},
    \end{equation}
    where $d(s_i, s_j)$ is the Euclidean distance between samples $s_i$ and $s_j$.
\end{deftn}

$\text{VR}_{\alpha}(\bS)$ is the set of all simplicial complexes $\{ s_0, \cdots, s_k \}$ where the pairwise distance $d(s_i, s_j)$ is within the threshold $\alpha$.
Let $C_k (\text{VR}_{\alpha}(\bS))$ denote the vector space generated by its $k$-dimensional simplices over $\mmZ_2$\footnote{It is not specific to $\mmZ_2$, but $\mmZ_2$ is a typical choice.}. The boundary operator $\partial_k: C_k (\text{VR}_{\alpha}(\bS)) \rightarrow C_{k-1} (\text{VR}_{\alpha}(\bS))$  maps each simplex to its boundary
, which consists of the sum of all its faces,
is a homomorphism from $C_k (\text{VR}_{\alpha}(\bS))$ to  $C_{k-1} (\text{VR}_{\alpha}(\bS))$. 
It can be shown that $ \partial_k \circ \partial_{k-1} = 0$, which leads to the chain complex: $ \cdots \rightarrow $
, and the $k^{\text{th}}$ homology group $H_{k}(\text{VR}_{\alpha}(\bS))$ is defined as the quotient group $H_{k}(\text{VR}_{\alpha}(\bS)):= \text{ker} \partial_k / \text{im} \partial_{k+1}$. ker represents kernel, which is the set of all elements that are mapped to the zero element. im represents image, which is the set of all the outputs. Rank $H_k(\text{VR}_{\alpha}(\bS))$ is known as the $k^{\text{th}}$ Betti number $\beta_k$, which counts the number of $k$-dimensional holes and can be used to represent the topological features of the manifold that the set of points $\bS$ sampled from.

However, the $H_{k}(\text{VR}_{\alpha}(\bS))$ is obtained based on a single $\alpha$, which is easily affected by small changes in $\bS$. Thus, it is not robust and is of limited use for real-world datasets. The persistent homology considers all the possible $\alpha$ instead of a single one, which results in a sequence of $\beta_k$. This is achieved through a nested sequence of simplicial complexes, called \textit{filtration}: $\text{VR}_{0}(\bS) \subseteq \text{VR}_{\alpha_1}(\bS) \subseteq \cdots \subseteq \text{VR}_{\alpha_m}(\bS)$ for $0 \leq \alpha_1 \leq \alpha_m $. 
Let $\gamma_i = [\alpha_i,\alpha_j]$ be the interval corresponding to a $k$-dimensional hole `birth' at the threshold $\alpha_i$ and `death' at the threshold $\alpha_j$, we denote $\text{PH}_{k}(\text{VR(\bS)}) = \{\gamma_i\}$ the set of `birth' and `death' intervals of the $k$-dimensional holes. We only exploit $\text{PH}_{0}(\text{VR(\bS)})$ in our PH-Reg, since we exploit the topological autoencoder as the topology part and higher topological features merely increase its runtime. An illustration of the calculation $\text{PH}_{0}(\text{VR(\bS)})$ is given in Figure ~\ref{fig_ph}. We define $E(\bS) =  \sum_{\gamma \in \text{PH}_0(\text{VR}(\bS))} |I(\gamma)|$, where $|I(\gamma)|$ is the length of the interval $\gamma$.

\begin{figure}[!t]
\centering
\subfigure[{Illustration of the framework}] {
\label{fig:framework}
\includegraphics[width=0.47\columnwidth]{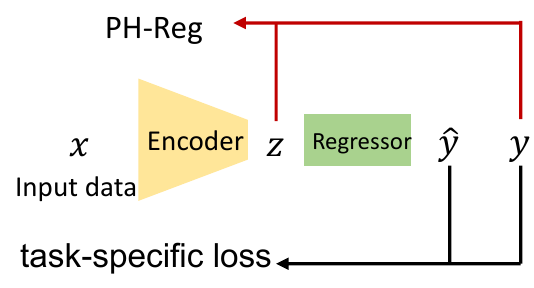}
}
\subfigure[Illustration of $\text{PH}_{0}(\text{VR(\bS)})$] {
\label{fig_ph}
\includegraphics[width=0.45\columnwidth]{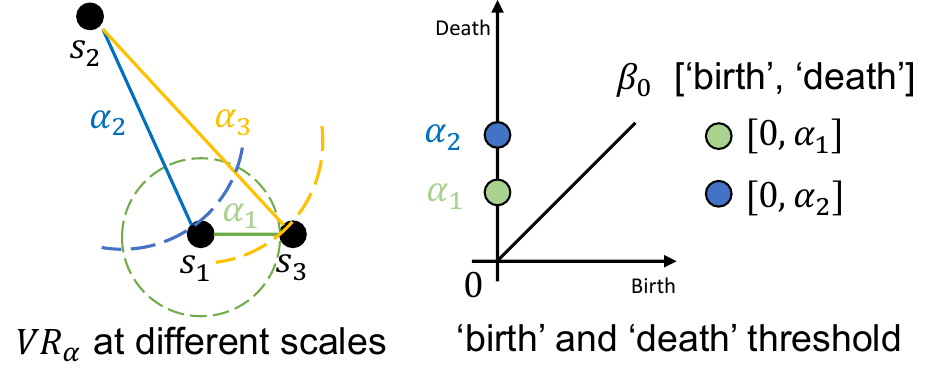}
}
\caption{Illustration of the (a) the use of PH-Reg for regression, and (b) calculating of $\text{PH}_{0}(\text{VR(\bS)})$. Here $\bS=\{s_1, s_2, s_3\}$. 
We say three connected components, \ie $\beta_0$, $(\{\{s_1\}, \{s_2\}, \{s_3\}\})$ `birth' when $\alpha=0$,  one  `death' (two left $(\{\{s_1, s_3\}, s_2\})$) when $\alpha=\alpha_1$, and another one `death' (one left $(\{\{s_1, s_3, s_2\})$) when $\alpha=\alpha_2$. Thus $\text{PH}_{0}(\text{VR(\bS)})=\{[0, \alpha_1], [0, \alpha_2] \}$. 
%When the threshold $\alpha$ increased from $0$ to $\alpha_1$ to $\alpha_2$, $\text{VR}_{\alpha}(\bS)$ contains $3(\{\{s_1\}, \{s_2\}, \{s_3\}\}), 2(\{\{s_1, s_3\}, s_2\})$ and $1(\{s_1, s_2, s_3\})$ connected components ($0^{\text{th}}$ Betti number $b_0$), respectively. We say the three connected components `birth' at $\alpha=0$, and $\{s_2\}$, $\{s_3\}$ `death' when $\alpha$ equals to $\alpha_1, \alpha_2$, respectively. Thus $\text{PH}_{0}(\text{VR(\bS)})=\{[0, \alpha_1], [0, \alpha_2] \}$.
%Similarly, $b_1=0, 1(\{\{s_1, s_3\}\}), 2(\{\{s_1, s_3\},\{s_1, s_2\}\}), 0$ when $\alpha$ increased from $0$ to $\alpha_1$ to $\alpha_2$ to $\alpha_3$, and thus $\text{PH}_{1}(\text{VR(\bS)})=\{[\alpha_1, \alpha_3], [\alpha_2, \alpha_3] \}$. 
%In this work, we only exploit $\text{PH}_{0}(\text{VR(\bS)})$.
}
\label{fig_weightentropy}
\end{figure}

\section{More details about the Coordinate Prediction task}
\label{appendix:syntheticDataset}

\textbf{Details about the synthetic dataset}:
We encode coordinates $\by \in \mmR^3$ into $100$ dimensional vectors $\bx_i = [f_1(\by_i), f_2(\by_i), f_3(\by_i), f_4(\by_i), \text{noise}]$, where the dimensions $1-4$ are signal and the rest $96$ dimensions are noise. The encoder functions $f_i$ are defined as:
\begin{itemize}
    \item $f_1(\by_i) = y_{i_1} + y_{i_2} +y_{i_3}$
    \item $f_2(\by_i) = y_{i_1} + y_{i_2} -y_{i_3}$
    \item $f_3(\by_i) = y_{i_1} - y_{i_2} +y_{i_3}$
    \item $f_4(\by_i) = -y_{i_1} + y_{i_2} +y_{i_3}$
\end{itemize}

As shown above, the accurate coordinates $\by_i$ can be obtained correctly when $f_1(\by_i), f_2(\by_i), f_3(\by_i), f_4(\by_i)$ are given. We introduce noise to the remaining $96$ dimensions by using $f_1, f_2, f_3, f_4$ on other randomly selected samples $\by_j$. The proximity of $\by_j$ to $\by_i$ can be intuitively seen as an indicator of the noise's relationship to the signal.

\textbf{More training details}: We train the models for $10000$ epochs using AdamW as the optimizer with a learning rate of $0.001$. We report results as mean $\pm$ standard variance over $10$ runs. For the regression baseline Oridnal Entropy and the IB baseline Information Dropout, we tried various weights $\{0.01, 0.1, 1, 10 \}$ and reported the best results. The trade-off parameters $\lambda_d$ and $\lambda_t$ are default set to $10$ and $100$, respectively, while $\lambda_t$ is set to $10000$ and $\lambda_d$ is set to $1$ for Mammoth, and $\lambda_d$ is set to 1 for torus and circle.

\section{Details about the real-world tasks}
\label{appendix:realworldTasks}

\subsection{Evaluation metrics}

\textbf{Depth Estimation.} We denote the predicted depth at position $p$ as $y_p$ and the corresponding ground truth depth as $y'_p$, the total number of pixels is $n$. 
The metrics are: 
1) threshold accuracy $\delta_1 \triangleq$ \% of $y_p, ~\mathrm{s.t.}~ \max(\frac{y_p}{y'_p}, \frac{y'_p}{y_p}) < t_1$, where $t_1=1.25$;
2) average relative error (REL): $\frac{1}{n}\sum_p\frac{|y_p - y'_p|}{y_p}$;
3) root mean squared error (RMS):  $\sqrt{\frac{1}{n}\sum_p(y_p - y'_p)^2}$;
4) average ($\log_{10}$ error): $\frac{1}{n}\sum_p|\log_{10}(y_p) - \log_{10}(y'_p)|$.

\textbf{Age Estimation.} Given $N$ images for testing, $y_i$ and $y'_i$ are the $i$-th prediction and ground-truth, respectively. The evaluation metrics include 
1)MAE: $\frac{1}{N}\sum_{i=1}^{N} |y_i - y'_i|$, and 
2)Geometric Mean (GM): $(\prod_{i=1}^{N}|y_i-y'_i|)^{\frac{1}{N}}$.

\subsection{Age estimation on AgeDB-DIR dataset}
We exploit the AgeDB-DIR~\citep{yang2021delving} for age estimation task. We follow the setting of  \citet{yang2021delving} and exploit their regression baseline model, which uses ResNet-50~\citep{he2016deep} as the backbone. 
The evaluation metrics are MAE and geometric mean(GM), and the results are reported on the whole set and the three disjoint subsets, \ie Many, Med. and Few. The trade-off parameters $\lambda_d$ and $\lambda_t$ are set to $0.1$ and $1$, respectively.
Table~\ref{appendix:agetable} shows that both $\mL_t$ and $\mL_d$ can improve the performance, and combining both achieves $0.48$ overall improvements (\ie ALL) on MAE and $0.25$ overall improvements on GM. 

\begin{table}[t!]
	\caption{Results on AgeDB. We report results as mean $\pm$ standard variance over $3$ runs.
	%Baseline and Baseline$^*$ denote the results from the original paper and our re-trained model. 
% 	\AY{\hl{ALL, Many, etc. refer to XX}}
	\textbf{Bold} numbers indicate the best performance.
	}
 \label{appendix:agetable}
	\centering
		\scalebox{0.7}{\begin{tabular}{c|cccc|cccc}
			\hline
			\multirow{2}[0]{*}{Method} & \multicolumn{4}{c|}{MAE~$\downarrow$} & \multicolumn{4}{c}{GM~$\downarrow$}  \\
 			\cline{2-9}
			& ALL & Many & Med. & Few & ALL & Many & Med. & Few\\
			\hline
			Baseline \cite{yang2021delving} & 7.80 $\pm$ 0.12 & 6.80 $\pm$ 0.06 & 9.11	$\pm$ 0.31 & 13.63	$\pm$ 0.43 & 4.98	$\pm$ 0.05 & 4.32	$\pm$ 0.06 & 6.19	$\pm$ 0.07 & 10.29	$\pm$ 0.57 \\
			 $+$ Information dropout \cite{achille2018information} & 8.04	$\pm$ 0.14 & 7.14	$\pm$ 0.20 & 9.10	$\pm$ 0.71 & 13.61	$\pm$ 0.32 & 5.11	$\pm$ 0.06 & 4.49	$\pm$ 0.17 & 6.14	$\pm$ 0.49 & 10.54	$\pm$ 0.65 \\
			 $+$ Oridnal Entropy \cite{zhang2023improving} & 7.65	$\pm$ 0.13 & 6.72	$\pm$ 0.09 & 8.77	$\pm$ 0.49 & 13.28	$\pm$ 0.73 & 4.91	$\pm$ 0.14 & 4.29	$\pm$ 0.06 & 6.04	$\pm$ 0.51 & 10.09	$\pm$ 0.62 \\
    \hline
% 			Baseline$^*$ & 7.73 & 6.74 & 9.05 & 13.35 & 5.05 & 4.38 & 6.44 & 9.77 \\
			$+ \mL'_d$ & 7.75	$\pm$ 0.05 & 6.80	$\pm$ 0.11 & 8.87	$\pm$ 0.05 & 13.61	$\pm$ 0.50 & 4.96	$\pm$ 0.04 & 4.33	$\pm$ 0.09 & 6.05	$\pm$ 0.36 & 10.43	$\pm$ 0.40\\
			$+ \mL_d$ & 7.64 $\pm$ 0.07 & 6.82	$\pm$ 0.07 & 8.62	$\pm$ 0.20 & 12.79	$\pm$ 0.65 & 4.85	$\pm$ 0.05 & 4.27	$\pm$ 0.06 & 5.91	$\pm$ 0.13 & 9.75	$\pm$ 0.53 \\
			$+ \mL_t$ & 7.50	$\pm$ 0.04 & 6.59	$\pm$ 0.03 & 8.75	$\pm$ 0.03 & 12.67	$\pm$ 0.24 & 4.77	$\pm$ 0.07 & 4.27	$\pm$ 0.06 & 6.09	$\pm$ 0.03 & 9.34	$\pm$ 0.70 \\
			$+ \mL_d + \mL_t$ &  \textbf{7.32}	$\pm$ \textbf{0.09} & \textbf{6.50}	$\pm$ \textbf{0.15} & \textbf{8.38}	$\pm$ \textbf{0.11} & \textbf{12.18}	$\pm$ \textbf{0.38} & \textbf{4.69}	$\pm$ \textbf{0.07} & \textbf{4.15}	$\pm$ \textbf{0.08} & \textbf{5.64}	$\pm$ \textbf{0.09} & \textbf{8.99}	$\pm$ \textbf{0.38} \\
			% $+ \mL_d + \mL_t$  & \textbf{7.46} & 6.73 & \textbf{8.18} & \textbf{12.38} & \textbf{4.72} & \textbf{4.21} & \textbf{5.36} & 9.70 \\
			\hline
		\end{tabular}}
\end{table}

\subsection{Super-resolution on DIV2K dataset}
% \label{subsection:superresolution}
We use the DIV2K dataset \citep{timofte2017ntire} for 4x super-resolution training (without the 2x pretrained model) and we evaluate on the validation set of DIV2K and the standard benchmarks: Set5 \citep{bevilacqua2012low}, Set14 \citep{zeyde2012single}, BSD100 \citep{martin2001database}, Urban100 \citep{huang2015single}. We follow the setting of  \citet{lim2017enhanced} and exploit their small-size EDSR model as our baseline architecture. We adopt the standard metrics PNSR and SSIM. The trade-off parameters $\lambda_d$ and $\lambda_t$ are set to $0.1$ and $1$, respectively.
Table \ref{tab:superresolution} shows that both $\mL_d$ and $\mL_t$ contribute to improving the baseline and adding both terms has the largest impact.

\begin{table}[t]
	\caption{Quantitative comparison of super-resolution results with public benchmark and DIV2K validation set. We report results as PSNR(dB)/SSIM.  \textbf{Bold} numbers indicate the best performance.}
	\centering
		\scalebox{0.8}{\begin{tabular}{c|ccccc}
			\hline
			\multirow{1}[0]{*}{Method}
			& Set5 & Set14 & B100 & Urban100 & DIV2K \\
			\hline
% 			Eigen et al. \citep{eigen2014depth} & 0.769 & 0.158 & 0.641 & -  \\
			Baseline \cite{lim2017enhanced} & 32.241/ 0.8656 & 28.614/ 0.7445 & 27.598/ 0.7120 & 26.083/ 0.7645 & 28.997/ 0.8189  \\
			$+$ Information dropout \cite{achille2018information} & 32.219/ 0.8649 & 28.626/ 0.7441 & 27.594/ 0.7113 & 26.059/ 0.7624 & 28.980/ 0.8182  \\
			$+$ Oridnal Entropy \cite{zhang2023improving} & 32.280/ 0.8653 & 28.659/ 0.7445 & 27.614/ 0.7119 & 26.117/ 0.7641 & 29.005/ 0.8188  \\
      \hline
			 $+ \mL'_d$ & 32.252/ 0.8653 & 28.625/ 0.7443 & 27.599/ 0.7118 & 26.078/ 0.7638 & 28.989/ 0.8186  \\
			 $+ \mL_d$ & 32.293/ 0.8660 & 28.644/ 0.7453 & 27.619/ 0.7127 & 26.151/ 0.7662 & 29.022/0.8197  \\
			 $+ \mL_t$ & \bf{32.322}/ 0.8663 & 28.673/ 0.7455 & 27.624/ 0.7127 & 26.169/ 0.7665 & 29.031/ 0.8196  \\
			 $+ \mL_d+ \mL_t$ & 32.288/ \bf{0.8663} & \bf{28.686}/ 0.7462 & \bf{27.627}/ 0.7132 & \bf{26.179}/ 0.7670 & \bf{29.038}/ 0.8201  \\
			\hline
		\end{tabular}}
\end{table}

% exploit EDSR \citep{lim2017enhanced} as our baseline architecture and follow the setting of the work

\subsection{Depth estimation on NYU-Depth-v2 dataset}
% \label{subsection:depth}
We exploit the NYU-Depth-v2 \citep{silberman2012indoor} for the depth estimation task. We follow the setting of \citet{lee2019big} and use ResNet50 as our baseline architecture.  We exploit the standard metrics of threshold accuracy $\delta_1, \delta_2, \delta_3$, average relative error (REL), root mean squared error (RMS)  and average $\log_{10}$ error.  The trade-off parameters $\lambda_d$ and $\lambda_t$ are both set to $0.1$. Table \ref{tab:nyu} shows that exploiting $\mL_t$ and $\mL_d$ results in reduction of $6.7\%$ and $8.9\%$ in the $\delta_1$ and $\delta_2$ errors, respectively. 

\subsection{Different improvement gap between synthetic and real-world datasets}
Our synthetic dataset is relatively simple and clean; the corresponding task (coordinate prediction) is directly related to the topology of the target space, hence the larger improvement on the synthetic data. The improvements on real-world datasets are significant (verified by Welch's t-test), and are also comparable to or better than competing works in the literature \cite{zhang2023improving, yang2021delving, lim2017enhanced, zhang2018residual}. Under the same settings, our improvements are competitive with recently published works. For age estimation on Age-DB, we improve the MAE (all) by almost 2-fold (0.48 vs. 0.25, see Table \ref{tab:appendix_age}). For super-resolution on DIV2K, we improve the PSNR (Set5, see Table \ref{tab:appendix_superresolution}) by 0.05, while typical state-of-the-art papers in super-resolution show increments of 0.01 on PSNR (Set5) \cite{lim2017enhanced, zhang2018residual}. For depth estimation, our improvements are on par with the competing method \cite{zhang2023improving} (REL: 0.144 vs 0.143, see Table \ref{tab:nyu}).

\begin{table}[t!]
	\caption{Quantitative comparison (MAE) on AgeDB. We report results as mean $\pm$ standard variance over $3$ runs.
	}
 	\label{tab:appendix_age}
	\centering
		\scalebox{1}{\begin{tabular}{c|ccccc}
			\hline
			\multirow{1}[0]{*}{Method}
			& ALL & Many & Med. & Few  \\
			\hline
			Baseline & 7.80 $\pm$ 0.12 & 6.80 $\pm$ 0.06 & 9.11	$\pm$ 0.31 & 13.63	$\pm$ 0.43  \\
			+ LDS \cite{yang2021delving} & 7.67 & 6.98 & 8.86 & 10.89  \\
			+ FDS \cite{yang2021delving} & 7.55 & 6.50 & 8.97 & 13.01  \\
			+ LDS + FDS \cite{yang2021delving} & 7.55 & 7.01 & 8.24 & 10.79  \\
			+ PH-Reg (ours) &  7.32	$\pm$ 0.09 & 6.50	$\pm$ 0.15 & 8.38	$\pm$ 0.11 & 12.18	$\pm$ 0.38  \\
   \hline
   + LDS + FDS vs. Baseline & + 0.25 & -0.19 & +0.97 & +2.94  \\
   + ous vs. Baseline & + 0.48 & +0.30 & +0.73 & +1.45  \\
   \hline
		\end{tabular}}
\end{table}

\begin{table}[t]
	\caption{Quantitative comparison (PSNR(dB)) of super-resolution results with public benchmark and DIV2K validation set. \textbf{Bold} numbers indicate the best performance.}
	\label{tab:appendix_superresolution}
	\centering
		\scalebox{1}{\begin{tabular}{c|ccccc}
			\hline
			\multirow{1}[0]{*}{Method}
			& Set5 & Set14 & B100 & Urban100 & DIV2K \\
			\hline
			Baseline (small-size EDSR \cite{lim2017enhanced}) & 32.24 & 28.61 & 27.60 & 26.08 & 29.00  \\
			EDSR \cite{lim2017enhanced} & 32.46 & 28.80 & 27.71 & 26.64 & 29.25  \\
			RDN \cite{zhang2018residual} & 32.47 & 28.81 & 27.72 & 26.61 & -  \\
			 Baseline + PH-Reg (ours) & 32.29 & 28.69 & 27.63 & 26.18 & 29.04  \\
      \hline
			 RDN vs. EDSR & +0.01 & +0.01 & +0.01 & -0.03 & -  \\
			 Ours vs. Baseline & +0.05 & +0.08 & +0.03 & +0.01 & +0.04  \\
			\hline
		\end{tabular}}
\end{table}

However, on NYU-Depth-V2, the representation of the head part (with many samples in the target space) is relatively well-learned, leading to a smaller impact on our regularizers. This might be a reason for the lower improvement. From the regression baseline (Figure \ref{fig_reg2}), the representation in the blue region (head part) already shows a lower intrinsic dimension and collapses like a line so there is little opportunity for our regularizers to have a strong impact (and hence the smaller improvement on the MAE (Many)). In contrast, the impact on the synthetic dataset and the green region (corresponding to the tail part, with limited samples in the target space) is more significant, resulting in a higher improvement.

Similar effects are observed on the other two real-world datasets, while the feature space on Age-DB tends to be a line (although the target space is discrete, it is too dense) and the feature space on DIV2K tends to be an object with a high density in the middle region (the target space is 3d).

\end{document}